\documentclass[authoryear]{elsarticle}
\journal{Games and Economic Behavior}
\usepackage[margin=1in]{geometry}
\usepackage[utf8]{inputenc}
\usepackage{graphicx}
\usepackage{subcaption}
\usepackage{amssymb}
\usepackage{booktabs}
\usepackage{graphicx}
\usepackage{amsmath}
\usepackage{amssymb}
\usepackage{dsfont}
\usepackage{bbold}
\usepackage{bbm}
\usepackage{nicefrac}
\usepackage{dsfont}
\usepackage{csquotes}
\usepackage{subcaption}
\usepackage{mathtools}
\usepackage{hyperref}
\usepackage{pgfplots}
\usepackage{verbatim}
\usepackage{float}
\usepackage{amsthm}
\usepackage{csquotes}
\usepackage{setspace}
\usepackage[UKenglish,USenglish]{babel}

\makeatletter
\def\ps@pprintTitle{%
  \let\@oddhead\@empty
  \let\@evenhead\@empty
  \let\@oddfoot\@empty
  \let\@evenfoot\@oddfoot
}
\makeatother

\graphicspath{\figures}

\newcommand{\Unif}{\operatorname{Unif}}

\newcommand{\1}{\mathds{1}}

\newcommand{\N}{\mathbb{N}}
\newcommand{\var}{\operatorname{var}}
\newcommand{\argmax}{\operatorname{arg\,max}}
\newcommand{\R}{\mathbb{R}}

\newcommand{\E}{\mathbb{E}}
\newcommand{\Exp}{\operatorname{Exp}}
\newcommand{\dd}{\,\mathrm d}
\renewcommand{\P}{\mathbb{P}}
\renewcommand\appendixautorefname[1]{}

\usepgfplotslibrary{groupplots,dateplot}
\usetikzlibrary{patterns,shapes.arrows}
\pgfplotsset{compat=newest}

\newtheorem{definition}{Definition}[]

\newtheorem{lemma}{Lemma}[]
\newtheorem{proposition}{Proposition}[]
\newtheorem{corollary}{Corollary}[]
\newtheorem{theorem}{Theorem}[]

\begin{document}
\begin{frontmatter}
\title{Risk Preferences of Learning Algorithms\tnoteref{label1}}
\tnotetext[label1]{We thank seminar audiences at Harvard and MIT and the mathoverflow.com user fedja for helpful conversations.}
\author{Andreas Haupt}
\ead{haupt@mit.edu}
\ead[url]{https://www.andyhaupt.com/}
\author{Aroon Narayanan}
\ead{aroon@mit.edu}
\address{Massachusetts Institute of Technology}
\ead[url]{https://sites.google.com/view/aroon-narayanan}
\begin{abstract}
Agents' learning from feedback shapes economic outcomes, and many economic decision-makers today employ learning algorithms to make consequential choices. This note shows that a widely used learning algorithm---$\varepsilon$-Greedy---exhibits emergent risk aversion: it prefers actions with lower variance. When presented with actions of the same expectation, under a wide range of conditions, $\varepsilon$-Greedy chooses the lower-variance action with probability approaching one. This emergent preference can have wide-ranging consequences, ranging from concerns about fairness to homogenization, and holds transiently even when the riskier action has a strictly higher expected payoff. We discuss two methods to correct this bias. The first method requires the algorithm to reweight data as a function of how likely the actions were to be chosen. The second requires the algorithm to have optimistic estimates of actions for which it has not collected much data. We show that risk-neutrality is restored with these corrections.
\end{abstract}
\end{frontmatter}
\section{Introduction}
Decision-makers often confront the same problem repeatedly, using the outcomes of their choices in the past to guide their understanding of the best course of action for the next time they encounter it. For example, credit scores are used to approve or deny credit based on the borrower's past credit behavior, and pretrial detention decisions are made based on the defendant's criminal history. Who gets the money to advance their lives---and who gets put in jail with curtailed liberties---will crucially depend on \textit{how} the prior data is used to make those decisions.

Many heuristics for solving these types of problems---essentially revolving around keeping an estimate of how well an action has performed in the past---have been developed. These have also been turned into formal algorithms deployed on computers and shown to have desirable properties. Such \enquote{learning algorithms} simultaneously make decisions and learn how to make future decisions better. They are now widely used in the economy, from product recommendation and pricing to highly consequential areas such as credit decisions and pretrial detentions.

While most deployed algorithms have a plethora of provable desirable properties, such as identifying the best option in the limit (\emph{no-regret}), we define and demonstrate an important bias in widely used learning algorithms: risk aversion, which emerges without being explicitly specified by the algorithm designer. For example, consider a repeated binary choice: Either pull lever A and get a deterministic payoff of $0$, or pull lever B and get a stochastic payoff of $1$ or $-1$, distributed uniformly. At any point in time, contingent on past observations of payoffs from the action taken (the \emph{bandit feedback setting}), an algorithm chooses an action, potentially randomly, to take next. What is the probability that the learning algorithm selects each of the actions after $t$ rounds of interaction? A risk-averse learning algorithm would choose the deterministic reward action more often than the other action. We prove that a classic algorithm, $\varepsilon$-Greedy, chooses the deterministic action with probability converging to $1$. This property holds without $\varepsilon$-Greedy being explicitly designed to be risk-averse: its design specification is to record the average reward received from each action and choose with some probability the action with the highest estimated average reward or any action at random. It is an emergent property of the algorithm.

Risk aversion of algorithms is more than a mere intellectual curiosity. It has stark implications across many real-world applications, particularly for fairness in algorithmic choice. This is especially true given that learning algorithms are increasingly being deployed for highly consequential societal decisions in the form of pretrial algorithms and risk assessment tools (RATs) and credit scoring algorithms, which necessitates extra attention to their emergent and unintended properties. In many of the economic settings where algorithms are deployed to make such decisions, a risk-averse algorithm can perpetuate deep inequities in society. As a concrete example, consider a firm making credit decisions using a risk-averse algorithm. Underrepresented minorities often have wide variances in their credit scores, shaped partly by inequities in access to good credit opportunities:
\hyphenblockcquote{UKenglish}[893]{baradaran2018jim}{\emph{As the white suburbs and black inner cities diverged in their mortgage access, two different credit markets emerged in both zones. Lower-risk mortgages led to higher wealth and stability in the white suburbs. These conditions also led to a
healthy consumer credit market. In the redlined black ghettos, the economic climate was radically different.}}.
When faced with minority applicants with higher variability in credit history, a risk-averse algorithm may decide to systematically deny them \emph{even when} it would have approved privileged applicants with similar expected repayment probability but features that are correlated with less variability in credit repayment, and hence perpetuate centuries of iniquity. Yet another setting in which it can play a significant role is recommendation systems, which regularly improve their recommendations using feedback from users. The choices that the recommendation algorithm takes, such as which content to show next to a user or which products to provide as an answer to a search query in e-commerce, are determined by how valuable the recommendation is deemed to be. Here, too, risk aversion can lead to the recommendation system suppressing \enquote{noisier} content---which, in most cases, will be the less mainstream, more marginalized content---even when its deployers do not find such bias desirable. In the long run, this bias can also lead to the homogenization of the content on these platforms, as more divergent content gets screened out by the algorithm's decision to not recommend it.

Our first formal result is that the $\varepsilon$-Greedy algorithm exhibits perfect risk-aversion under some conditions: as $T \to \infty$, it chooses the riskless action with probability one. The proof exposes that the mechanism that leads to algorithmic risk preferences is related to how it estimates the value of each action. If an algorithm's estimate of each action is the simple average of the observed rewards in the past from these actions, which is what $\varepsilon$-Greedy does, its estimates will be biased because the algorithm undersamples actions after bad reward draws. We then discuss two corrections to the algorithm that enable it to be risk-neutral. Our first correction, which we call the \emph{Reweighted $\varepsilon$-Greedy}, counters the undersampling propensity by adjusting its estimate to account for the probability with which the action was chosen in the first place. We show that Reweighted $\varepsilon$-Greedy is risk neutral. We also propose another correction for a broader class of settings: the \emph{Optimistic $\varepsilon$-Greedy}. It introduces an optimism term to the estimate that corrects for the pessimism that bad draws induce, in the style of the celebrated Upper Confidence Bound (UCB) algorithm (refer \cite{auer2002using,auer2002finite,Lattimore2020}). Our third formal result shows that this correction also makes $\varepsilon$-Greedy risk neutral. We use simulations to explore the necessity of conditions we make in our theoretical analysis and the transient persistence of risk behavior even with unequal expected values for the actions. 

\subsection*{Related literature} \label{sec:litrev}

There is a large literature on learning in economics, compare, e.g., \cite{bolton1999strategic},  \cite{keller2005strategic}, \cite{klein2011negatively}; see \cite{bergemann2006bandit} for an early survey. The closest in spirit is \cite{bardhi2020early}, in which the authors show that even arbitrarily small differences in early-career discrimination can be highly consequential later on. Our results complement this literature by showing that algorithmic learning can exhibit unintended discrimination with strong consequences in the long run. A second branch of literature is empirical. Papers such as \cite{farber1996learning} and \cite{altonji2001employer} study learning by employers, showing that it has testable implications for wage dynamics. \cite{crawford2005uncertainty} applies learning to demand for pharmaceutical drugs. Recently, there has also been a strand that empirically analyzes the behavior of learning algorithmic, particularly in relation to collusion. \cite{calvano2020artificial}, \cite{musolff2022algorithmic}, and \cite{brown2023competition} find in different settings that pricing algorithms learn to play collusive equilibria, raising antitrust concerns about the use of such algorithms. Such emergent behavior in a game theoretic setting closely resembles the emergent preference we illustrate in an algorithmic decision-theoretic setting.

Our paper is connected to the computer science literature on the effect of biased payoff estimates in recommendation systems. \cite{causalrec} observes that online learning in recommendation systems leads to confounding of average user scores in recommendation systems and proposes algorithmic interventions to correct this bias. \cite{Chaney2018} proposes a model and shows that recommendation systems' biased estimates of user preferences can increase homogeneity and decrease user utility. Our study highlights the fact of noise in reward distributions as a reason for biased estimates instead of taking bias as a given.

We also relate to the study of fairness in bandit problems. While \cite{joseph2016fairness} considers fairness (which is equivalent to our notion of risk neutrality) as a constraint for algorithm design and constructs algorithms that (approximately) satisfy it, this paper provides evidence on the risk preferences of an existing algorithm, $\varepsilon$-Greedy, and proposes two ways to mitigate risk preferences. Consider also \cite{patil2021achieving} and \cite{liu2017calibrated} for treatments of fairness in bandit problems.

Finally, we relate to the regret analysis of bandit algorithms under diffusion scaling. \cite{kalvit2021closer} studies this for the Upper Confidence Bound algorithm (compare \autoref{thm:optimistic}). \cite{fan2021diffusion} derives the limit action distribution of Thompson sampling as a solution to a random ordinary differential equation.

\subsection*{Outline} 
The structure of the rest of this note is as follows. In \autoref{sec:model}, we introduce our online learning setup and our definition of risk aversion, along with formal definitions of our algorithms. The result on $\varepsilon$-Greedy's risk aversion is presented in \autoref{sec:epsgreedy}. We discuss two corrections of risk aversion in \autoref{sec:corr}. We complement our theoretical analysis with simulations in \autoref{sec:simulations}.

\section{Model}\label{sec:model}
In a bandit problem, a decision maker repeatedly takes an action from a finite set $A$, $\lvert A \rvert < \infty$. Each action $a$ is associated to a sub-Gaussian distribution $F_a \in \Delta (\R), a \in A$ with expectation $\mu_a$ and variance proxy $\sigma_a^2$.\footnote{A distribution is sub-Gaussian if $\int e^{\lambda X} \dd F_a(x) \le \exp(\frac{\lambda^2 \sigma_a^2}{2})$. In this case, $\sigma_a^2$ is called the \emph{variance proxy}.} A \emph{strategy} or \emph{algorithm} used by the decision maker is a function $\pi \colon \bigcup_{t=1}^\infty (A \times [0,1])^t \to \Delta(A)$. We denote action-reward histories by $(a_{1:t}, r_{1:t})$ and the probability that action $a \in A$ is chosen in round $t$ by $\pi_{at} = \pi(a_{1:t}, r_{1:t})_a$.

An algorithm $\pi$ generates (potentially random) sequences of \emph{actions} $(a_t)_{t \in \N}$ and \emph{rewards} or \emph{payoffs} $(r_t)_{t \in \N}$. For each $t\in \N$, repeatedly, the algorithm chooses an action $a_t \sim \pi_{t}$ and receives a reward $r_t \sim F_{a_t}$. Denote $N_a(t) \coloneqq \lvert \{ 1 \le t' \le t : a_t = a\}\rvert$ the number of times action $a$ has been chosen up to time $t$.

The main concept in this paper is a notion of \emph{risk-aversion} of algorithms. An algorithm is risk-neutral if it chooses (asymptotically) uniformly from amongst actions of equal expectation. In the long run, risk-averse algorithms prefer less risky actions than others of the same expectation. In the extreme case where the algorithm exclusively chooses (asymptotically) the least risky action among those of the same expectation, we call them \emph{perfectly} risk averse.

\begin{definition}
We call $\pi$ \emph{risk-neutral} if for any actions $a, a' \in A$ such that $\mu_a = \mu_{a'}$,
\[
\lim_{t \to \infty} \P[a_t = a] = \lim_{t \to \infty} \P[a_t = a']
\]
We call an algorithm \emph{risk-averse} if for all $a,  a' \in A$ such that $\mu_a = \mu_{a'}$ and $F_a \prec_\text{SOSD} F_{a'}$, it holds that
\[
\lim_{t \to \infty}\P [a_t = a] > \lim_{t \to \infty} \P [a_t = a'].\footnote{Distribution $F$ dominates $F'$ in second-order stochastic dominance, $F \succeq_{\text{SOSD}} F'$ if $\int u \dd F \ge \int u \dd F'$ for all concave, non-decreasing functions $u$, with a strict inequality for some such function $u$.}
\]
An algorithm is \emph{perfectly risk averse} if whenever there exists an $a \in A$ such that either $\mu_{a'} < \mu_a$ or $F_{a'} \prec_{\text{SOSD}} F_a$ for all actions $a ' \in A \setminus \{a\}$, it is true that
\[
\lim_{t \to \infty}\P [a_t = a] =1.
\]
\end{definition}
This paper considers the $\varepsilon$-Greedy algorithm and two variants of $\varepsilon$-Greedy with different sufficient statistics.
\begin{definition}[$\varepsilon$-Greedy] Let $(\varepsilon_t)_{t \in \N}$ be a $[0,1]$-valued sequence. $\varepsilon$-Greedy chooses the empirically best action with probability $1-\epsilon_t$, and randomizes between all the actions with probability $\epsilon_t$, i.e.
\begin{equation*}
    \pi_{at} = \begin{cases}
  \Unif (\argmax\limits_{a \in A} \mu_a(t-1) )& \text{ w.p. }1-\varepsilon_t \\  \Unif (A)  & \text{ w.p. } \varepsilon_t,
    \end{cases}
\end{equation*}
where $\mu_a(t-1)$ is historical average payoff
\[
\mu_a(t) \coloneqq \frac{1}{N_a(t)}\sum_{\substack{1 \le t ' \le t \\ a_{t'} = a}} r_{t'}.
\]
When $\varepsilon$-Greedy takes an action to maximize $\mu_a(t-1)$, we say it \emph{exploits} or \emph{takes an exploitation step}. Otherwise, it \emph{explores}. We also call $\mu_a(t-1)$ $\varepsilon$-Greedy's \emph{sufficient statistic}.
\end{definition}

The first variant reweighs data points to change their importance.

\begin{definition}[Reweighted $\varepsilon$-Greedy]
Reweighted $\varepsilon$-Greedy uses a reweighted payoff estimate as a sufficient statistic:
\begin{equation*}
    \mu_{a,r}(t) =\frac{1}{N_a(t) }\sum_{\substack{1 \le t' \le t\\a_t= a}} \frac{r_{t'}}{\sqrt{\pi_{at'}}}.
\end{equation*}
\end{definition}

A second intervention adds an optimism term to the sufficient statistic of $\varepsilon$-Greedy.

\begin{definition}[Optimistic $\varepsilon$-Greedy]
Optimistic $\varepsilon$-Greedy uses sufficient statistic $\mu_{a,o}$, such that
\begin{equation*}
    \mu_{a,o}(t) = \mu_a(t) + \rho \sqrt{\frac{\log(t)}{N_a(t)}}.
\end{equation*}
If action $a$ has not been chosen before, $\mu_{a,o}(t-1) = \infty$.
\end{definition}

\section{Risk aversion of epsilon-Greedy}\label{sec:epsgreedy}
We first show that $\varepsilon$-Greedy exhibits perfectly risk-averse behavior.
\begin{theorem}\label{thm:epsgreedy}
Let $(\varepsilon_{t})_{t \in \N}$ such that $\varepsilon_t \to 0$ and $\sum_{t=1}^\infty \varepsilon_t = \infty$. If there is a deterministic, centered dominant action, and all other actions have symmetric continuous distributions, $\varepsilon$-Greedy is perfectly risk-averse.
\end{theorem}

A discussion on the conditions in this result is in order before we move into the proof. The learning rates are both necessary to yield a \emph{no-regret algorithm}, compare \cite{Lattimore2020}, and hence are rather innocuous. We show in our simulations in \autoref{sec:simulations} that relaxing the requirement of determinism of a dominant action leads to risk aversion, but not perfect risk aversion, as does relaxing the symmetry requirement. Our result also holds for Rademacher-distributed action rewards.

The intuition behind this result lies in the sampling bias of $\varepsilon$-Greedy. Upon receiving a low payoff for an action, it becomes less likely to choose that action and, hence, less likely to receive data to correct its estimate. This means that it keeps a pessimistic estimate of reward. This bias leads to behavior that is consistent with risk aversion.

\begin{proof}[Proof of Theorem \ref{thm:epsgreedy}]
We first observe that we can restrict to bandit problems of two actions with actions of the same expectation, one deterministic dominant action $a$ and another action $a'$ with a continuous, symmetric distribution. Under the assumptions on the learning rate made in the theorem, $\varepsilon$-Greedy chooses dominated actions with vanishing probability. We prove this in Lemma \ref{lem:epsgreedynoregret} in the appendix. As this probability is low, we may consider instances of actions of equal expectation. In addition, a union bound shows that if the probability that the algorithm chooses action $a$ over any single action $a'$ converges to $1$, this implies that this action will be chosen with probability one among all actions of the same probability. Hence, it is without loss to restrict to two-action bandit problems.

We also observe that the result is trivial for two deterministic actions with the same expectation. Hence, we may assume that $\var(F_a)$ has a positive variance. Furthermore, it is without loss to assume that the deterministic action is centered: $\varepsilon$-Greedy is invariant to the addition of a constant to all reward distributions.

As a final reduction step, as $\varepsilon_t \to 0$, it is sufficient to show that it becomes unlikely that $\varepsilon$-Greedy chooses $a'$ in exploitation steps, or
\[
\P[\mu_a(t) < \mu_{a'}(t)] \to 0.
\]
We express this event as a property of a stochastic process. As one of the actions in the $\varepsilon$-Greedy is deterministic and (without loss of generality) centered, the sum of the payoffs of the non-deterministic action, denoted by $(X_t)_{t \in \N}$, is a sufficient statistic for the dynamics of the algorithm. $(X_t)_{t \in \N}$ is a lazy random walk starting at the origin, $X_0 = 0$, with transition kernel
\[
X_{t+1} = \begin{cases}
    X_t + r & \text{ with probability $(1-\frac{\varepsilon}{2}) 1_{X_t > 0} + \frac12 1_{X_t = 0} + \frac{\varepsilon}{2} 1_{X_t < 0}$,}\\
    X_t & \text{else,}
\end{cases}
\]
where $r \sim F_a$. We call this the \emph{advantage walk} and depict it in \autoref{fig:cover}.
\begin{figure}
\begin{subfigure}{.5\textwidth}
  \centering
  \includegraphics[width=.95\linewidth]{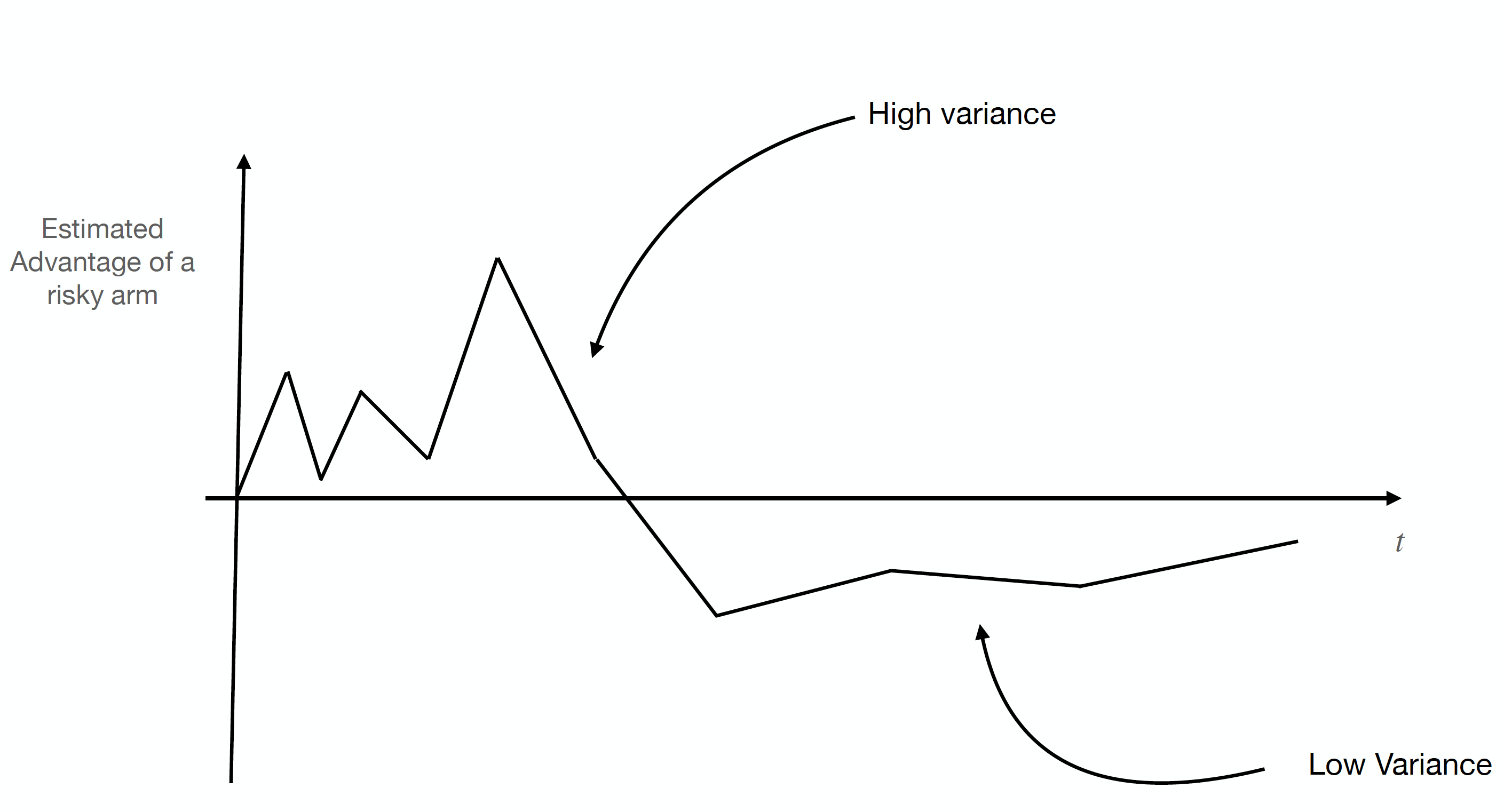}
  \caption{A stylized depiction of an advantage walk.}
  \label{subfig:stylizedqvalues}
\end{subfigure}%
\begin{subfigure}{.5\textwidth}
  \centering
  \includegraphics[width=.95\linewidth]{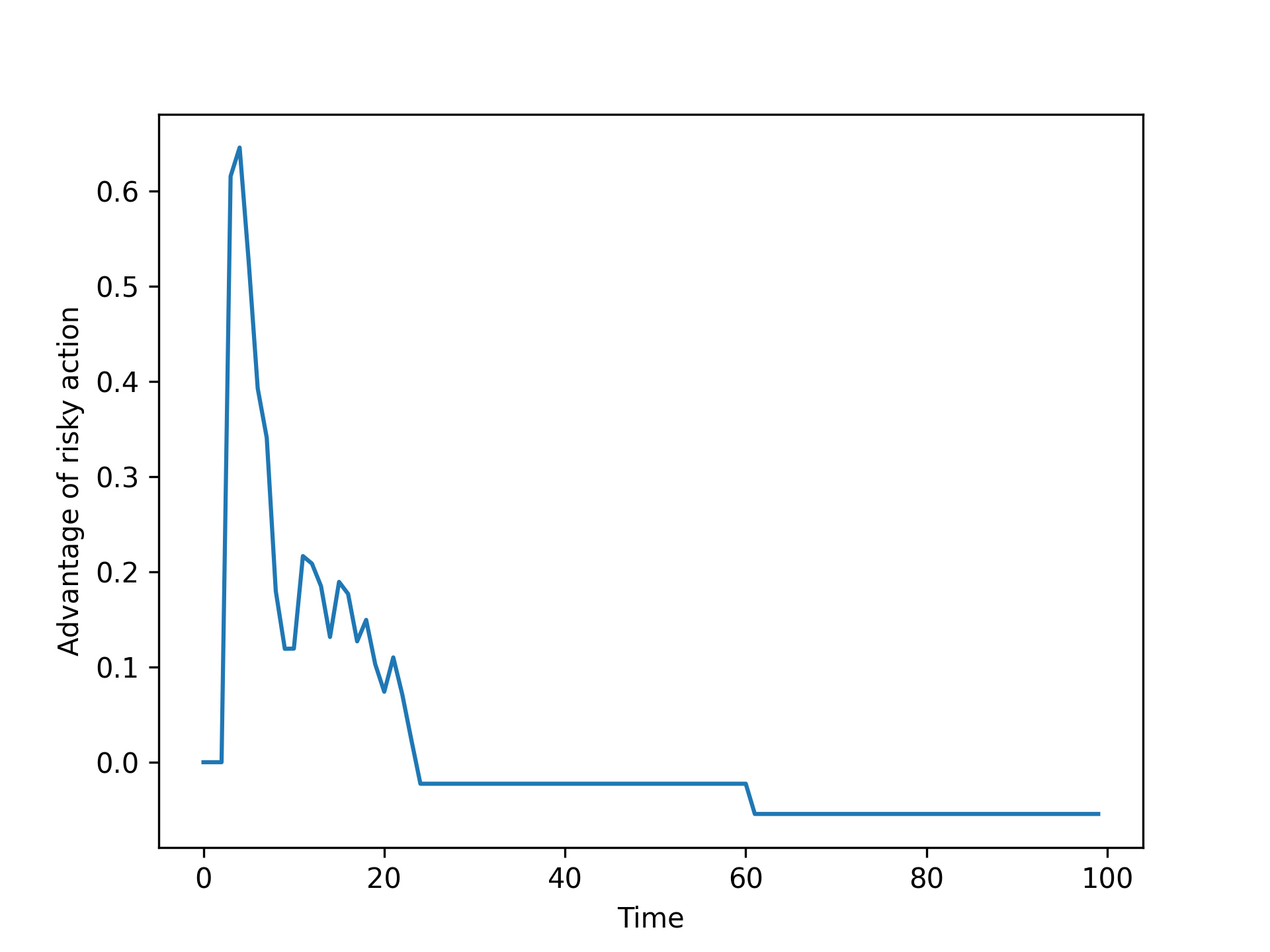}
  \caption{One realization of the advantage walk for $\varepsilon$-Greedy where the safe action has distribution $\mathbb{1}_{\{0\}}$ while the risky action has distribution $U[-1, 1]$}
  \label{subfig:realqvalues}
\end{subfigure}
\caption{The advantage walk for $\varepsilon$-Greedy. The main intuition for risk aversion in online algorithms is that a random walk with non-uniform variance spends more time in places with less variance.}
\label{fig:cover}
\end{figure}
Consider the time since the last time that the random walk crossed zero, 
\[
\tau_t \coloneqq t - \max \{2 \le t' \le t | X_{t'-1} \le 0 \le X_{t'}\}.
\]
We claim that $\tau_t \xrightarrow[t \to \infty]{\P} \infty$.

To prove this claim, observe that for any $c \ge 0$, there is $C > 0$ and $t \in \N$ such that for all $t' \ge t$
\[
\P[ \tau_{t'} \le c] \le \P[ \lvert X_{t' - c} \rvert < C] + \varepsilon \le 2 \varepsilon.
\]
The first inequality is a consequence of the sub-Gaussianity of $F_a$. The second inequality is a result of $\var(F_a) > 0$, the conditional independence of increments, and $\sum_{t'= t-c}^t \varepsilon_{t'} \to \infty$ as $t \to \infty$.

We also define the distribution of the number of steps taken since $\tau_t$ on the positive side. These are distributed as
\begin{align*}
P_t &\sim \sum_{t' = t - \tau_t}^t Z_{t'}, \quad Z_t \sim \operatorname{Bern}(1-\varepsilon_t/2).
\end{align*}
As $P_{\tau_t}$ is a sum of i.i.d. random variables, by Hoeffding's inequality, $P_{\tau_t}$ is close to its conditional expectation $\E[P_{\tau_t} | \tau_t]$ with high probability in $\tau_t$ and hence with high probability in $t$. This conditional expectation is
\[
\E[P_{\tau_t}| \tau_t] = \sum_{t'= t- \tau_t} 1 - \frac{\varepsilon_{t'}}{2}.
\]
Hence, for any $\delta>0$ that there is $t \in \N$ such that for all $t' \ge t$
\begin{equation*}
\begin{split}
\P[X_{t'} > 0 | \tau_t ] &\ge \P[Y_1, Y_2, \dots, Y_{P_{t'}} > 0 | \tau_t].
\end{split}
\end{equation*}
where $Y_0 = 0$ and $(Y_t)_{t \in \N}$ is a standard random walk with increment distribution $F_a$. 

Hence, for any $\delta>0$, there is $t' \in \N$ such that for all $t' \ge t$,
\begin{equation}
\begin{split}
\P[X_{t} > 0 | \tau_{t}] &\le c\P[Y_1, Y_2, \dots, Y_{P_{t} + 1} > 0] \\
&\le  \frac{c}{\sqrt{\pi(\sum_{t'= t- \tau_t} 1 - \frac{\varepsilon_{t'}}{2})}} (1+\delta).
\end{split}\label{eq:lengthbound}
\end{equation}
For the first inequality, we can choose $c \ge 1/\P[ r > X_{t' - \tau_{t'}}]$, where $r \sim F_a$ is independent of $(X_t)_{t \in \N}$. This follows as a single step from zero could lead from $0$ to $X_{t'}$, or a higher value. Because $X_{t' - \tau_{t'}}$ is reached from $X_{t' - \tau_{t'} - 1} < 0$, $\P[ r > X_{t' - \tau_{t'}}]> 0$ must be positive, and hence $c$ is well-defined. 

The second inequality uses the well-known property that the probability of a random walk stays positive until time $t$ with probability approximately $\frac{1}{\sqrt{\pi t}}$ \cite[Eqn. 35]{frisch1995universality}.\footnote{This property also holds for Rademacher-distributed increments and is the only place where we use the assumption that our distribution is continuous and symmetric.}

Observe that \eqref{eq:lengthbound} approaches $0$ as $\tau_t \to \infty$ and recall that $\tau_t \xrightarrow[t \to \infty]{\P} \infty$. Given these two facts,
\[
\P[X_t > 0] \xrightarrow[t \to \infty]{\P} 0.
\]
As $\P[X_t = 0] \to 0, t \to \infty$, this concludes the proof.
\end{proof}

We highlight that a similar argument applies to actions that are dominated by others in expectation. For two actions $a, a'$ of the same expectation, such that $a$ is deterministic but $a'$ is not, and a third action $a''$ such that $\mu_{a''} > \mu_a, \mu_{a'}$, the probability that $a$ is taken conditional  $a$ or $a'$ are taken converges to one.

\section{Achieving risk-neutrality}\label{sec:corr}
Next, we propose variants to the $\varepsilon$-Greedy sufficient statistic to achieve risk neutrality. These corrections are motivated by the technical analysis of the previous theorem as well as known algorithmic ideas (see ex \cite{auer2002using}), which enabled us to pinpoint the reason why it was exhibiting risk aversion, but crucially, we are also able to link them directly to insights on how policymakers should address possible biases in deployed algorithms. Thus, these \enquote{corrected} algorithms should also be seen as clarifying the underlying mechanisms and their redressal. 

\subsection{A reweighting approach to risk-neutrality}
The first approach stems from our analysis of variance: we should reweight data to achieve risk neutrality. A reweighted $\varepsilon$-Greedy provably is risk neutral if exploration is sufficiently high.
\begin{theorem}[Reweighted $\varepsilon$-Greedy]\label{thm:reweighted}
Let $\varepsilon_t =t^{\frac12 + \kappa}, \kappa \in (0, \frac12)$. Reweighted $\varepsilon$-Greedy is risk-neutral for two centered actions, one of which is deterministic.
\end{theorem}
The reweighting proposed here means that payoffs resulting from currently unfavoured actions are weighted more highly in the sufficient statistic of the algorithm. In the credit scoring setting, if the option of rejecting the loan is currently favored, the outcome of any credit that is given (due to exploration) is weighted more highly. Thus, this correction tells us that we should assign more importance to outcomes that resulted from choices that seemed apriori less attractive.

It is worth noting that this reweighting does \emph{not} lead to an unbiased estimator of action rewards. We show in simulations in an appendix that an unbiased estimator leads to a risk-affine algorithm, \autoref{sec:add_sim}. This is a result of what the rescaling does to the internal sufficient statistics: The goal of the algorithm proposed here is to equalize variance, which is at odds with producing an unbiased estimator.

Several comments on the assumptions are in order. The first assumption on exploration means that a sufficient amount of exploration is needed for this algorithm to be risk-neutral. We provide a simulation in \autoref{sec:simulations} showing that this condition is necessary for risk neutrality. The assumption of centeredness and determinism is needed for our proof, but risk aversion seems to hold beyond them, as we show in \autoref{sec:simulations}. On the other hand, the requirement that there are only two actions with the same expectation is crucial, as we show in another simulation in \autoref{sec:simulations}.
\begin{proof}
Note that if both actions are deterministic, the conclusion of the theorem holds trivial. Otherwise, denote the non-deterministic action by  $a \in A$. The evolution of choices can be expressed only based on the stochastic process 
\[
    Y_t = \sum_{\substack{1 \le t' \le t\\ a_t = a}}\frac{r_t}{\sqrt{\pi_t}}.
\]
If $Y_t > 0$, action $a$ is chosen with probablity $1 - \varepsilon_t/2$, if $Y_t = 0$, it is chosen with probability $1/2$, and if $Y_t < 0$, then it is chosen with probability $\varepsilon/2$. We define the random array
\[
X_{t't} = \frac{1}{\sqrt{t}} Y_{t'}
\]
This random array has the following properties:

\begin{description}
\item[Martingale]  It is an $L^2$-martingale array, i.e. $(X_{t't})_{1 \le t' \le t}$ is a square-integrable martingale with respect to its natural filtration.
\item[Asymptotic Variance] The conditional variances of martingale increments are constant (and deterministic).
\begin{align*}
\sum_{t' = 1}^t \E [ (X_{t't} - X_{(t'-1)t})^2 | X_{(t'-1)t}] &= \sum_{t'=1}^t \sum_{a \in A} \pi_{a(t'+1)} \E\left[ \frac{r_{a}^2}{\sqrt{t\pi_{a(t'+1)}}^2} \middle | X_{t't}\right] \\
&= \sum_{t'=1}^t \frac1t \sum_{a \in A} \sigma_a^2 \\
&= \sigma^2_a.
\end{align*}
In particular, as $n \to \infty$, $\E [ (X_{t't} - X_{(t'-1)t})^2 | X_{(t'-1)t}] \to  \sigma_a^2$ in probability. 
\item[Lindeberg Condition] For any $\varepsilon > 0$, we have that
\begin{align*}
\sum_{t'=1}^t \E[(X_{t't} - X_{(t' - 1)t})^2 \1_{\lvert X_{t't} - X_{(t'-1)t}\rvert \ge \varepsilon} | X_{(t'-1)t}]
&\le 2\sum_{t'=1}^t t^{-1}(t')^{1 - 2\kappa} \E[r^2 \1_{r \ge t^{-\frac12}(t')^{\frac12 - \kappa}\varepsilon}] \\
&= \frac{2}{t} \sum_{t'=1}^t (t')^{1 - 2\kappa} \E[r^2 \1_{r \ge t^{\frac12}(t')^{\frac12 - \kappa}\varepsilon}] \\
&\le  \frac{2}{t} \sum_{t'=1}^t (t')^{1 - 2\kappa} \sigma_a^2 e^{\frac{\lambda^2 \sigma_a^2}{2} - \lambda t^{\frac12}(t')^{\frac12 - \kappa}\varepsilon}\\
& \le \frac{2}{t} \sum_{t'=1}^t t^{1 - 2\kappa} \sigma_a^2 e^{\frac{\lambda^2 \sigma_a^2}{2} - \lambda t^{\frac12}\varepsilon}\\
& \to 0.
\end{align*}
The first inequality uses the fact that we can bound 
\[
\E[(X_{t't} - X_{(t' - 1)t})^2 | X_{(t'-1)t}] \le \frac{\E[r^2]}{t\min_{a \in A} \pi_{at'}} \le \frac{2\E[r^2]}{t\varepsilon_{t'}} = \frac{2\sigma_a^2}{t\varepsilon_{t'}}.
\]
The second is arithmetic. The third applies a Chernoff bound. The last uses $1 \le t \le t'$. The convergence follows as exponential decay dominates polynomial growth and as convergence of a sequence implies convergence of the Cesàro mean.
\end{description}
Given these conditions, we can apply a Martingale Central Limit Theorem  \cite[Corollary 3.1]{hall2014martingale} and conclude that the distribution of $X_{tt}$ converges to $N(0,\sigma_a^2)$. This means that
\[
\P[X_{tt} > 0], \P[X_{tt} < 0] \xrightarrow{t \to \infty} \frac12,
\]
and hence
\[
\P[a_t = a] = \frac12.
\]
\end{proof}
While this algorithm works for two actions and, as we show, works for two actions, another approach to risk neutrality, \emph{optimism}, allows us to guarantee risk neutrality for an arbitrary number of actions of equal expectation. This is what we discuss in the next subsection.

\subsection{An optimism approach to risk-neutrality}
Another way to modify the sufficient statistic is not to reweight but to explicitly favor alternatives that have not previously been chosen as frequently in the past. Conventionally, this is referred to as \emph{optimism} in the multi-armed bandit literature, compare \cite[Section 1.3.3]{slivkins2019introduction}. We show that it ensures risk-neutrality.

\begin{theorem}[Optimistic $\varepsilon$-Greedy]\label{thm:optimistic}
There exists $\rho_0 > 1$ such that for any $\rho \ge \rho_0$ and any $(\varepsilon_{t})_{t \in \N}$ with $\varepsilon_t \to 0$, Optimistic $\varepsilon$-Greedy is risk-neutral.
\end{theorem}
\begin{proof}[Proof Sketch]
Note first that for non-exploration steps of Optimistic $\varepsilon$-Greedy, the policy is the same as Upper Confidence Bound with exploration coefficient $\rho$, compare \cite{auer2002using}. We adapt the proof of \cite[Theorem 2]{kalvitneurips} for our variant of $\varepsilon$-Greedy. Theorem 2 in \cite{kalvitneurips} shows that an optimistic policy without exploration has the property that 
\[
\lim_{t \to \infty} \P[a_t = a] \to \frac{1}{\lvert\argmax_{a \in A} \mu_a \rvert}
\]
for all $a$ such that $a \in \argmax_{a \in A} \mu_a$ and $\rho > \rho_0$. As Optimistic $\varepsilon$-Greedy does not incur regret as we show in \autoref{sec:proof}, this property implies that the algorithm does not incur regret.

The full proof can be found in \cite[Appendix D]{kalvit2021closer}. The proof goes as follows. First, show that $\frac{N_i(t)}{t} > \frac{1}{2\lvert\argmax_{a \in A} \mu_a \rvert}$ with probability approaching $1$ in the limit, i.e. the fraction of times any action is chosen can be bounded below in probability. This is proved by means of a union bound and a Hoeffding bound. The operative equation that gets to this lower bound is \cite[Eqn. 40]{kalvit2021closer}, which depends on \cite[Eqns. 35 and 39]{kalvit2021closer}. Using this lower bound, we show that $\lvert N_i(t) - N_j(t) \rvert $ is small, i.e. the difference in the number of times any two actions are chosen is small as the time goes to infinity. This again uses Markov's inequality, along with the Law of the Iterated Logarithm. Now, consider optimistic $\epsilon$-Greedy. Since it implements the Upper Confidence Bound policy with probability $1-\epsilon_t$ and randomizes otherwise, in either case, it must be eventually choosing uniformly from amongst the highest mean actions, except for the vanishing probability with which it chooses dominated actions.
\end{proof}

It is worth noting why the mathematical intuition from the earlier result on $\varepsilon$-Greedy breaks. In this proof, the main object was a random walk with different variances for positive and negative sufficient statistics. Optimism may be seen as introducing a drift towards the origin. The proof shows that this drift is strong enough to correct risk aversion. We also note that this approach works for non-centered random variables. 

This result demonstrates that one way to build a fairer world is to be as optimistic as is rationally possible. Linking back to the credit decisions example we referenced in the introduction, credit scoring algorithms should evaluate applicants in the best possible light, adjusting for the risk profile of minority applicants by accounting for the fact that they often have lower access to good credit opportunities.

\section{Simulations}\label{sec:simulations}
For the final section of this note, we use simulations to investigate how far our results extend beyond the conditions we have imposed in our theoretical results. 

\paragraph{On the conditions of \autoref{thm:epsgreedy}} Our first set of simulations, shown in \autoref{fig:epsilon-greedy}, relate to the conditions in \autoref{thm:epsgreedy}. \autoref{subfig:perfriskaversiona} shows that within the conditions of \autoref{thm:epsgreedy}, $\varepsilon$-Greedy exhibits perfect risk aversion. \autoref{subfig:perfriskaversionb} shows that the assumption of symmetry is not necessary, as even for asymmetric distributions such as the exponential distribution, risk aversion holds, and \autoref{subfig:perfriskaversionnoncentered} shows that it also holds for distributions that are not centered around $0$. If there is no optimal deterministic action, we show in \autoref{fig:three_arms} that $\varepsilon$-Greedy is risk-averse yet not perfectly risk-averse. Hence, we find that the risk aversion result is relatively robust.

\paragraph{On the conditions of \autoref{thm:reweighted}} Next, \autoref{fig:reweighted} shows that the results of \autoref{thm:reweighted} extend beyond its assumptions, but only as long as exploration is sufficiently high. The restriction to two actions is necessary, as \autoref{fig:reweightedthreearms} shows.

\paragraph{On the conditions of \autoref{thm:optimistic}} Finally, we consider the assumptions of \autoref{thm:optimistic}. In \autoref{fig:optimistic}, we see that Optimistic $\varepsilon$-Greedy is quite risk neutral as long as $\rho$ is sufficiently high, which is in line with \autoref{thm:optimistic}. When $\rho$ is low, the level of optimism is insufficient to counteract $\varepsilon$-Greedy's pessimism. 

\paragraph{Finite-time behavior for dominated actions} We show in \autoref{fig:better_risky} that $\varepsilon$-Greedy's risk aversion has large transient effects before asymptotic guarantees kick in. This clarifies that the bias we identify can persist for a long time---for example, credit decisions can continue to be significantly discriminatory with such an algorithm \textit{even when} the minority candidate has a strictly higher likelihood of repaying the loan.

\begin{figure}[ht]
\begin{subfigure}{.32\textwidth}
\centering
\includegraphics[width=.95\linewidth]{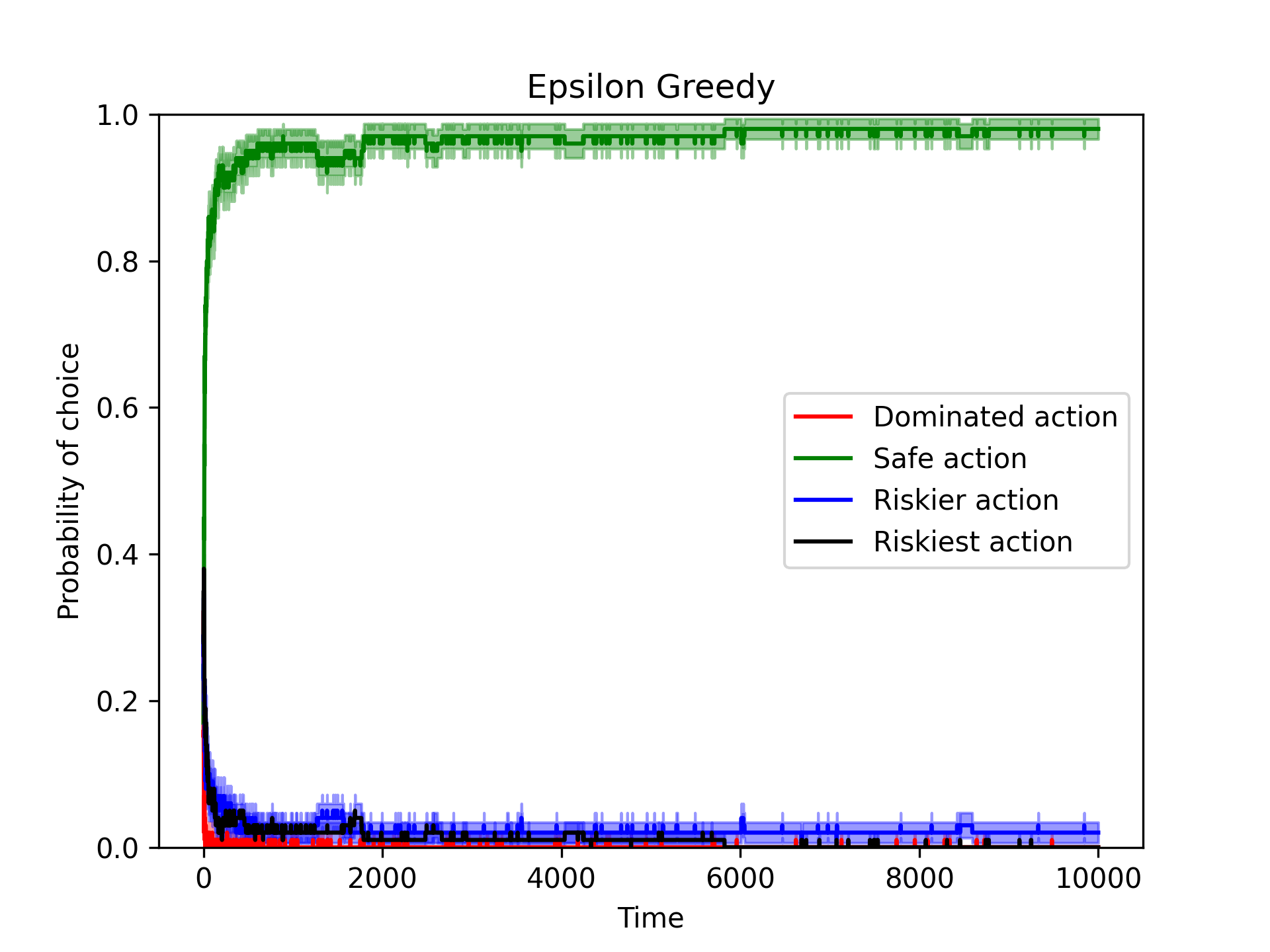}
\caption{Perfect risk aversion.}
\label{subfig:perfriskaversiona}
\end{subfigure}%
\begin{subfigure}{.32\textwidth}
\centering
\includegraphics[width=.95\linewidth]{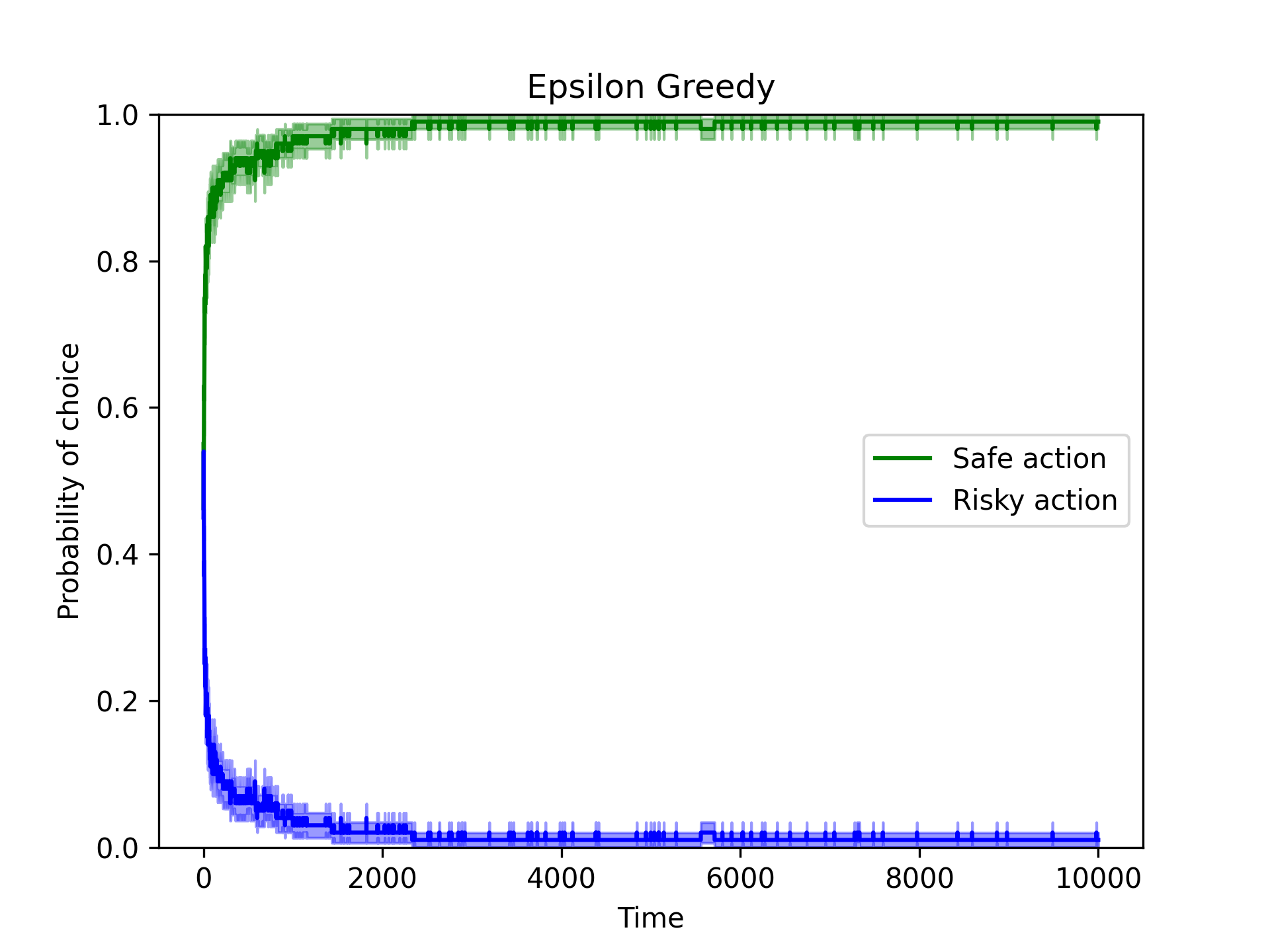}
\caption{Perfect risk aversion.}
\label{subfig:perfriskaversionb}
\end{subfigure}%
\begin{subfigure}{.32\textwidth}
\centering
\includegraphics[width=.95\linewidth]{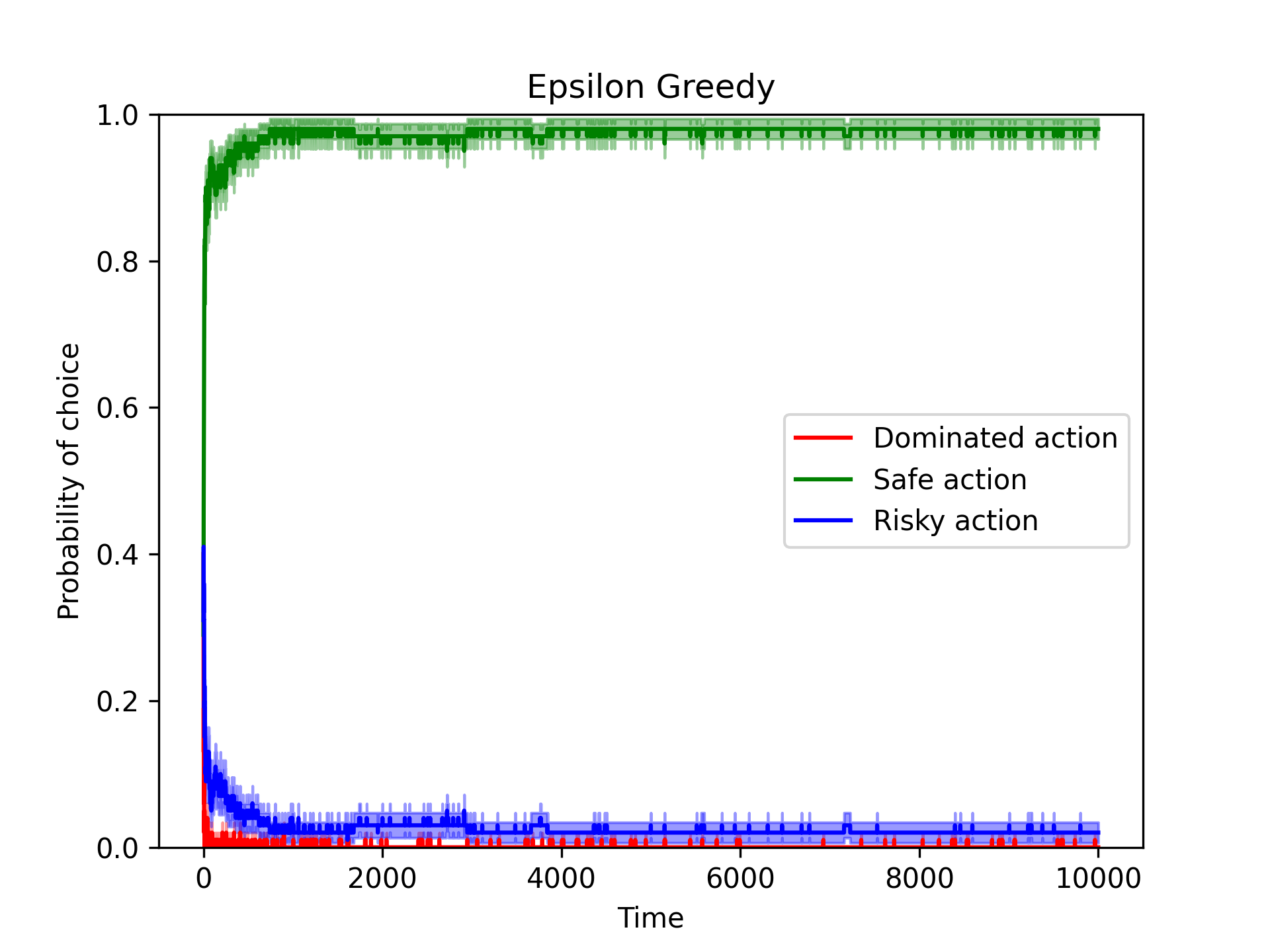}
\caption{Perfect risk aversion.}
\label{subfig:perfriskaversionnoncentered}
\end{subfigure}%
\caption{Plots of the behaviour of $\varepsilon$-Greedy, under and outside the conditions of \autoref{thm:epsgreedy}. In all plots, the dominated action has payoff distribution $\mathbb{1}_{\{-1\}}$. In (a), safe action has distribution $\mathbb{1}_{\{0\}}$, the riskier action has payoff distribution $U[-0.5, 0.5]$ while the riskiest action has payoff distribution $U[-1, 1]$. In (b) safe action has distribution $\mathbb{1}_{\{10\}}$ while the risky action has distribution an exponential distribution with rate 10 i.e. $\Exp(10)$. In (c), the safe action has distribution $\mathbb{1}_{\{0.5\}}$ while the risky action has distribution $U[-1, 2]$. We set $\varepsilon_t = t^{-1}$. Note that the bands around the plot line are $90\%$ confidence intervals, with 100 independent runs.}
\label{fig:epsilon-greedy}
\end{figure}

\begin{figure}[ht]
\begin{subfigure}{.32\textwidth}
\centering
\includegraphics[width=.95\linewidth]{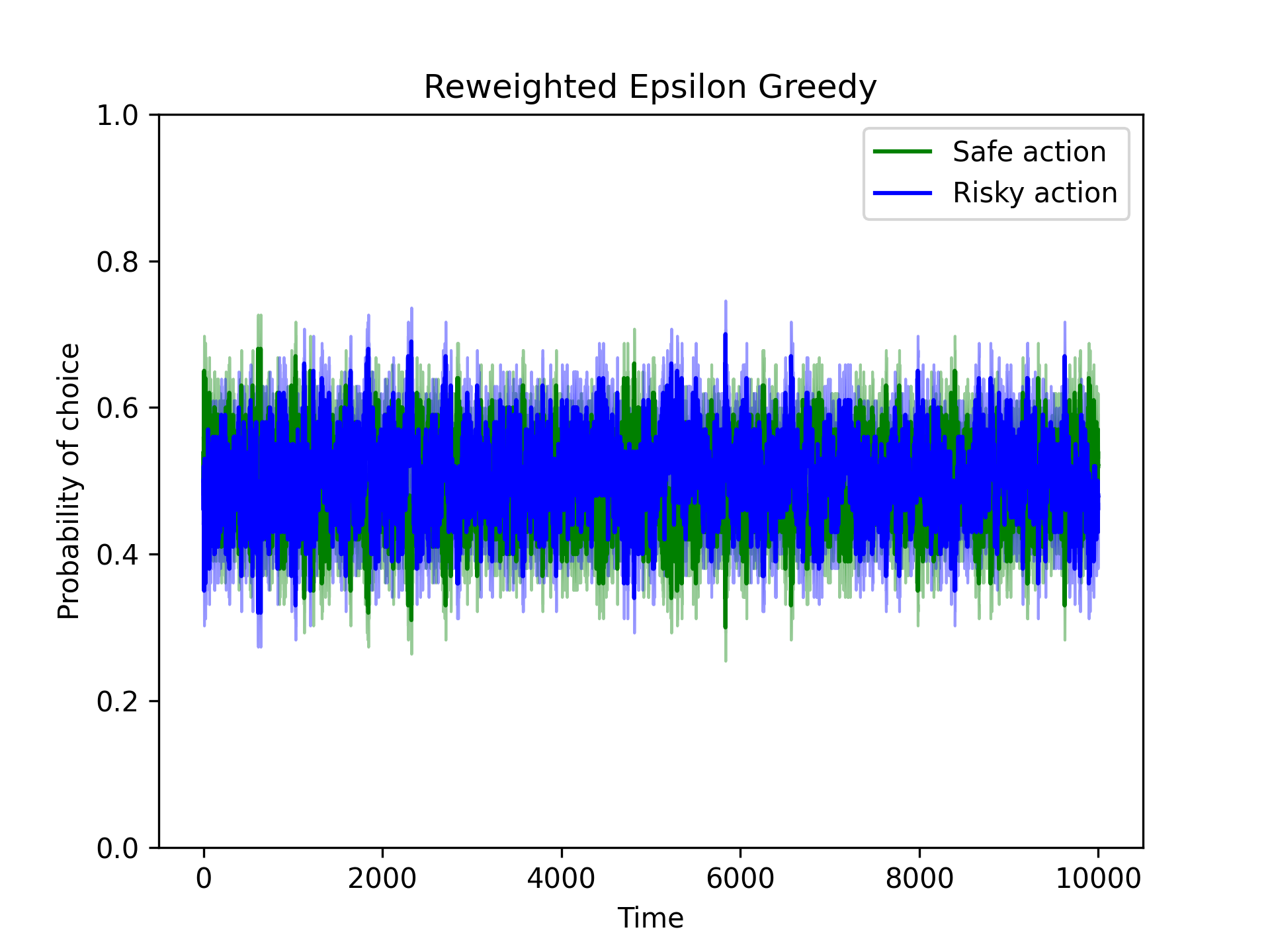}
\caption{Risk neutral.}
\label{fig:reweightedcentered}
\end{subfigure}%
\begin{subfigure}{.32\textwidth}
\centering
\includegraphics[width=.95\linewidth]{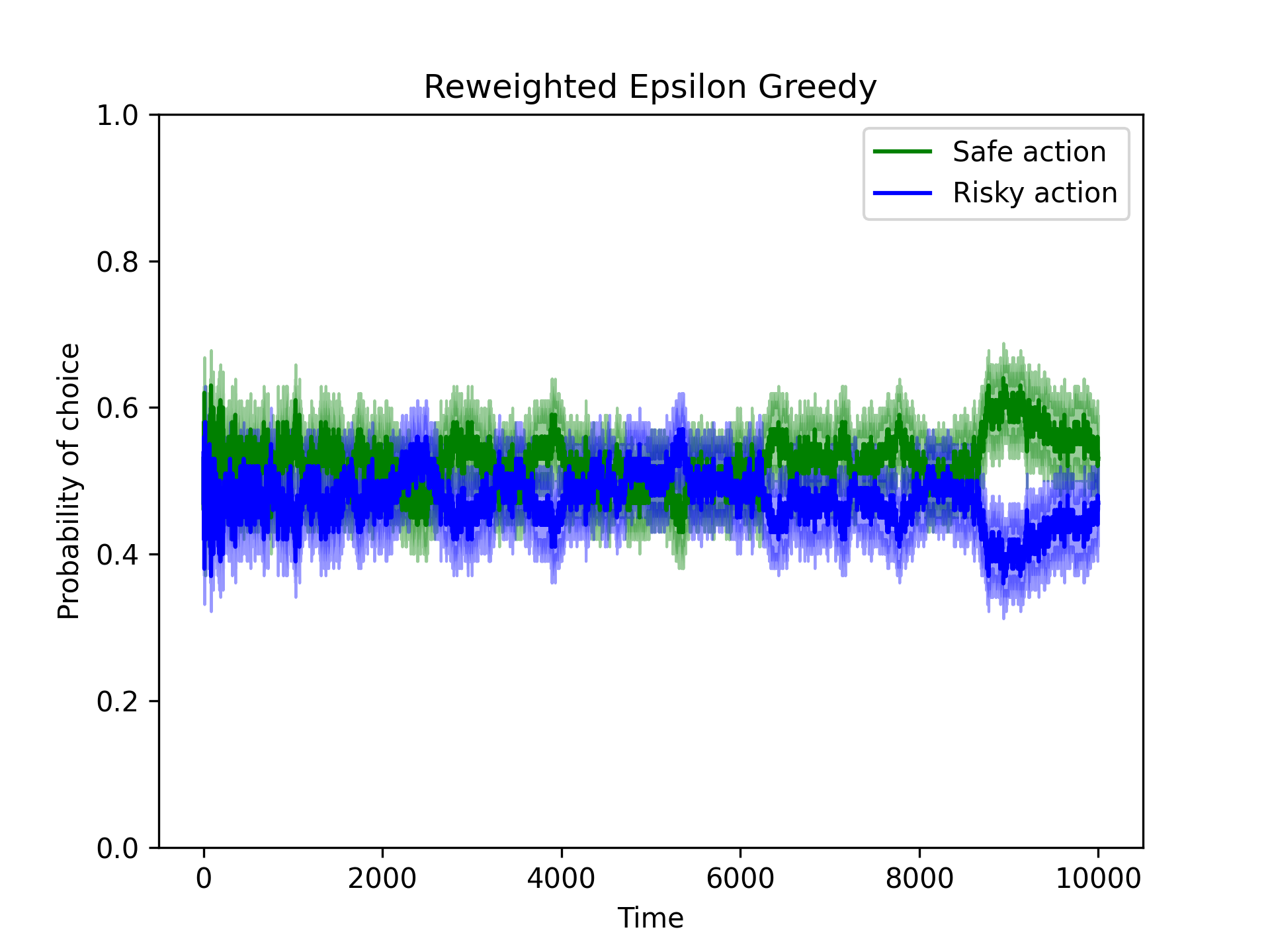}
\caption{Risk neutral.}
\label{fig:reweightednoncentered}
\end{subfigure}
\begin{subfigure}{.32\textwidth}
\centering
\includegraphics[width=.95\linewidth]{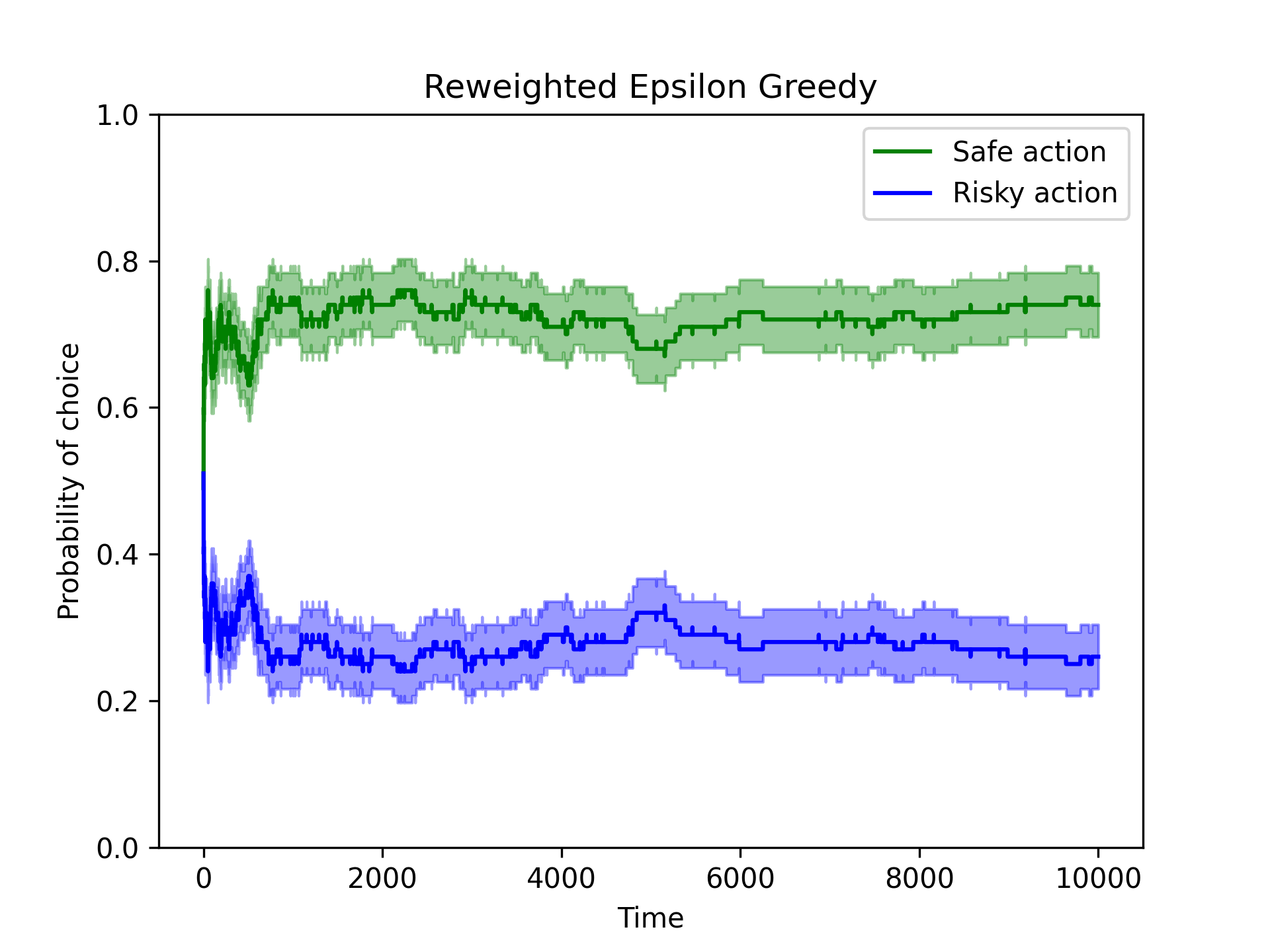}
\caption{Risk averse.}
\label{fig:reweightedlowexpl}
\end{subfigure}
\caption{Plots showing that the Reweighted $\varepsilon$-Greedy is risk-neutral for two actions as long as exploration is sufficiently high. In (a), the safe action has payoff distribution $\mathbb{1}_{\{0\}}$ while the risky action has payoff distribution $U[-1, 1]$. In (b), the safe action on the left has a distribution $U[0.25, 0.75]$ while the risky action has a payoff distribution $U[0, 1]$. We set $\varepsilon_t = t^{-0.49}$ in (a) and (b). In (c), we use the same setup as in experiment (a), but with an exploration rate of $\varepsilon_t = t^{-1}$. Note that the bands around the plot line are $90\%$ confidence intervals, with 100 independent runs.}
\label{fig:reweighted}
\end{figure}

\begin{figure}[ht]
\begin{subfigure}{.32\textwidth}
\centering
\includegraphics[width=.95\linewidth]{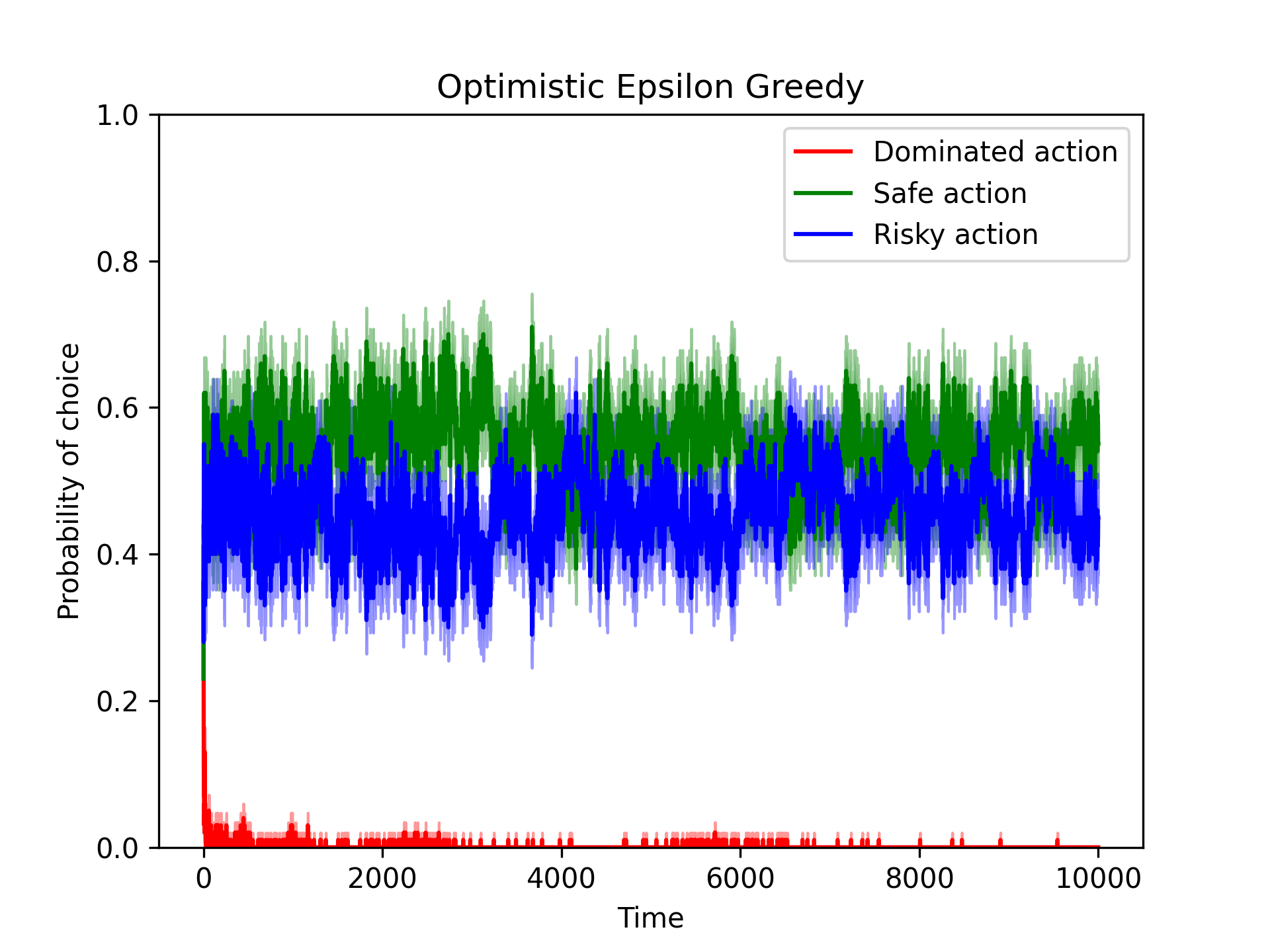}
\caption{Risk neutral.}
\label{fig:riskneutralcentered}
\end{subfigure}%
\begin{subfigure}{.32\textwidth}
\centering
\includegraphics[width=.95\linewidth]{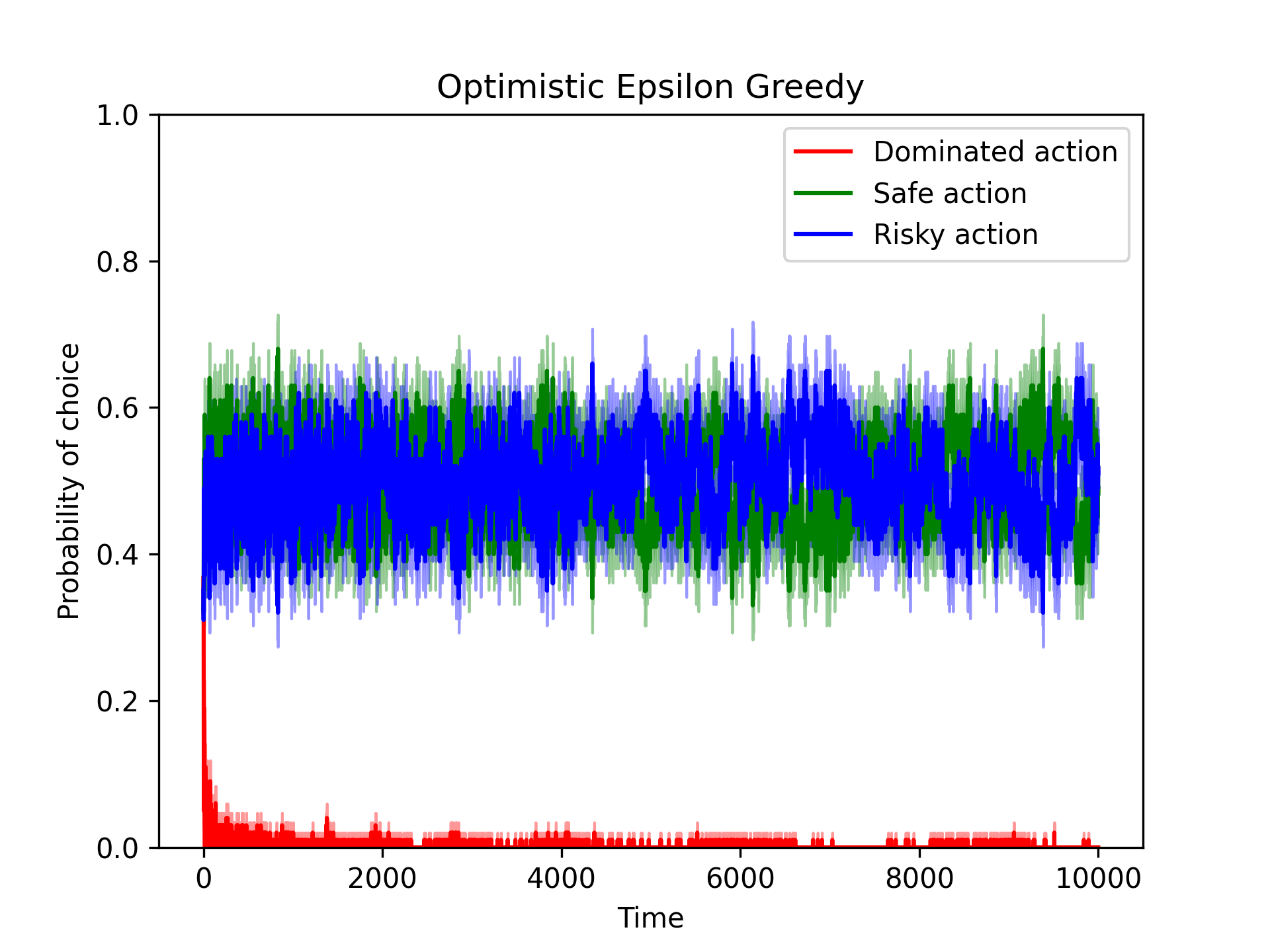}
\caption{Risk neutral.}
\label{fig:riskneutralnoncentered}
\end{subfigure}
\begin{subfigure}{.32\textwidth}
\centering
\includegraphics[width=.95\linewidth]{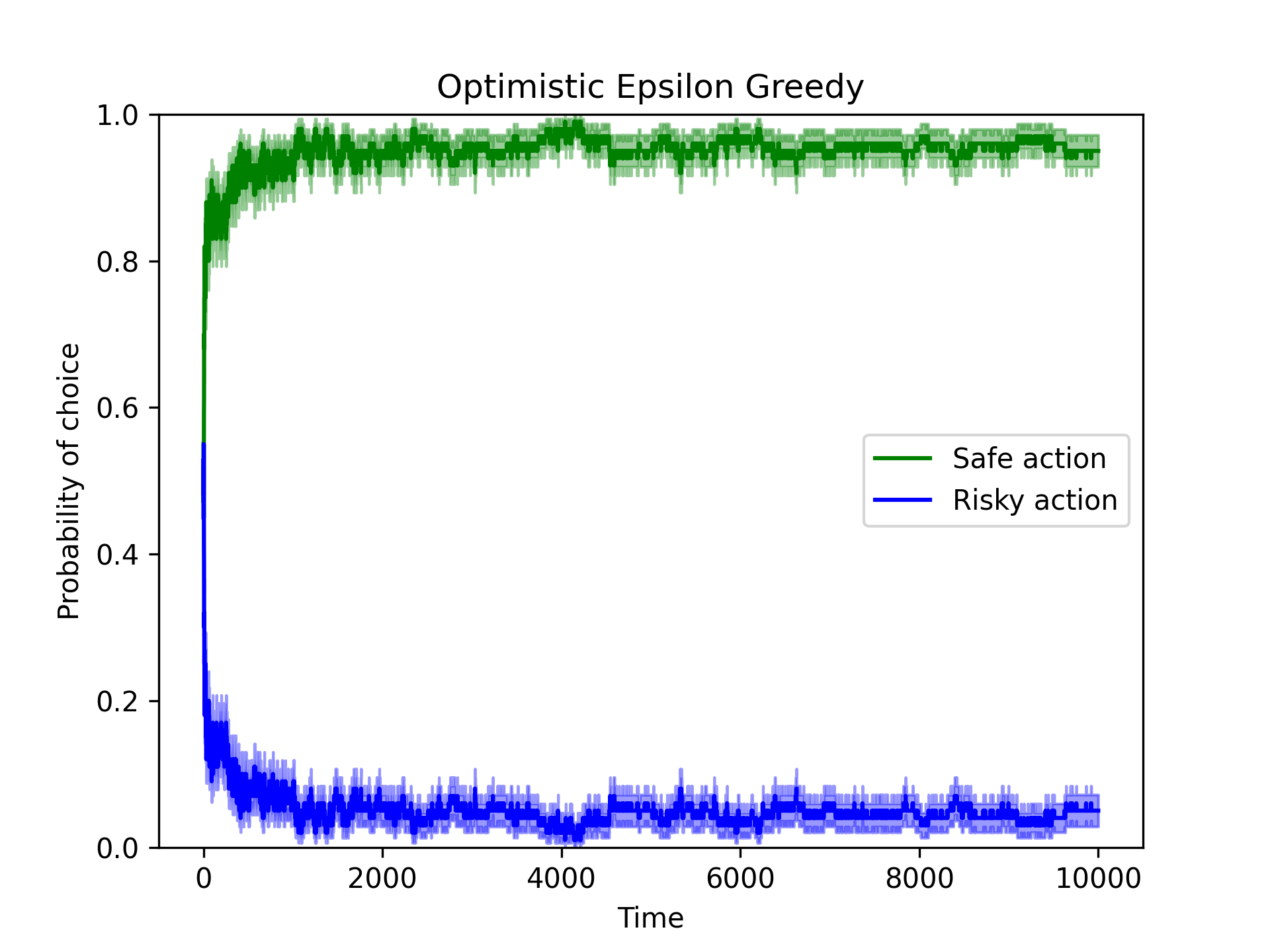}
\caption{Risk averse.}
\label{fig:riskaversesmallrho}
\end{subfigure}
\caption{Plots showing that the Optimistic $\varepsilon$-Greedy is quite generally risk-neutral, as long as exploration coefficient $\rho$ is sufficiently high. In all plots, the dominated action has payoff distribution $\mathbb{1}_{\{-1\}}$. In (a), the safe action has a payoff distribution $U[-0.25, 0.25]$ while the risky action has a payoff distribution $U[-1, 1]$. In b), the safe action on the left has a distribution $U[0.25, 0.75]$ while the risky action has a payoff distribution $U[0, 1]$. Both (a) and (b) have $\rho = 2$. In (c), the safe action has payoff distribution $\mathbb{1}_{\{0\}}$, and the risky action has payoff distribution $U[-1, 1]$, with $\rho = 0.02$. We set $\varepsilon_t = t^{-1}$. Note that the bands around the plot line are $90\%$ confidence intervals, with 100 independent runs.}
\label{fig:optimistic}
\end{figure}

\begin{figure}[ht]
\centering
\begin{subfigure}{.34\textwidth}
\centering
\includegraphics[width=.95\linewidth]{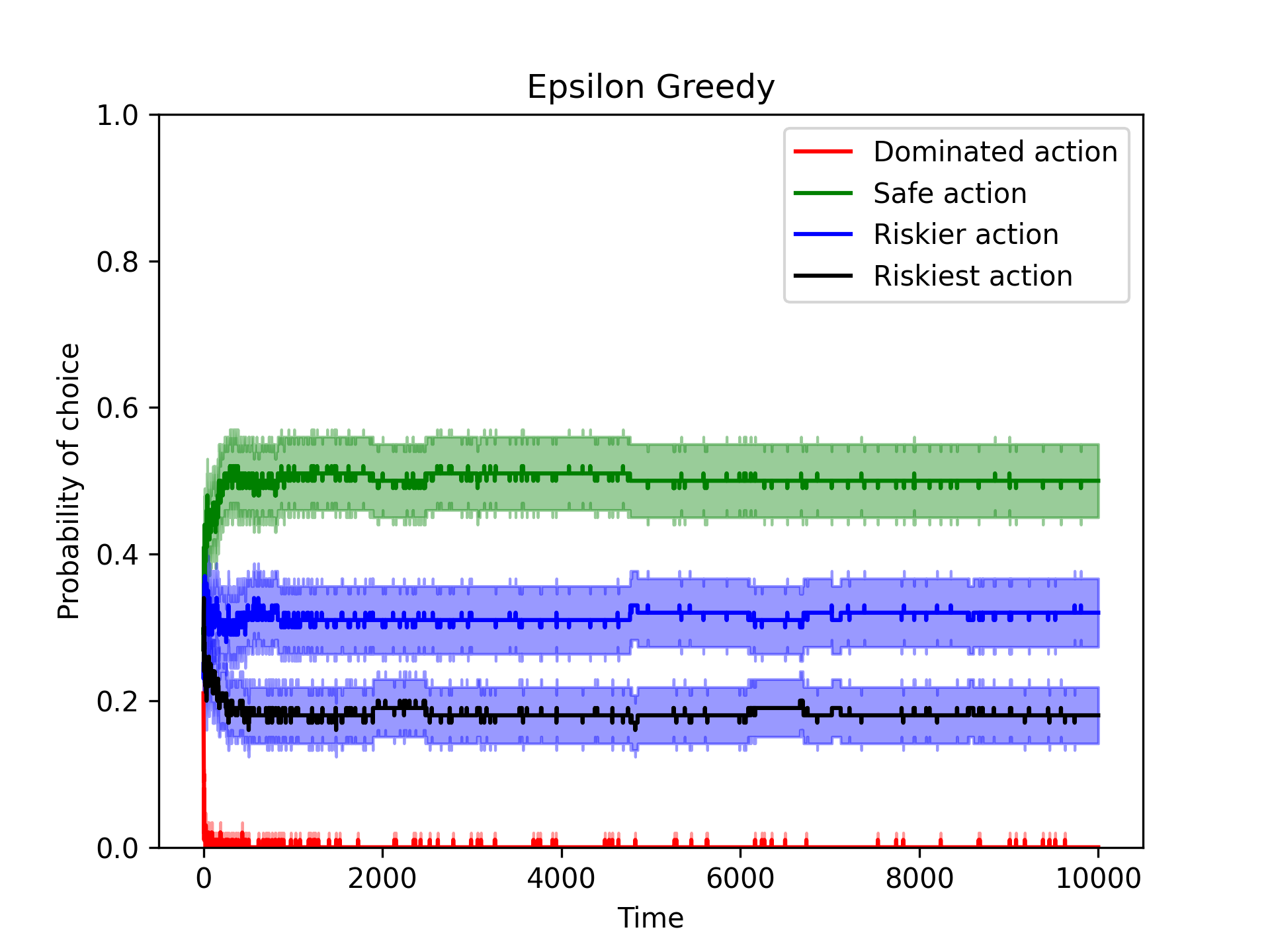}
\caption{$\varepsilon$-Greedy with no optimal safe action.}
\label{fig:three_arms}
\end{subfigure}%
\begin{subfigure}{.34\textwidth}
\includegraphics[width=.95\linewidth]{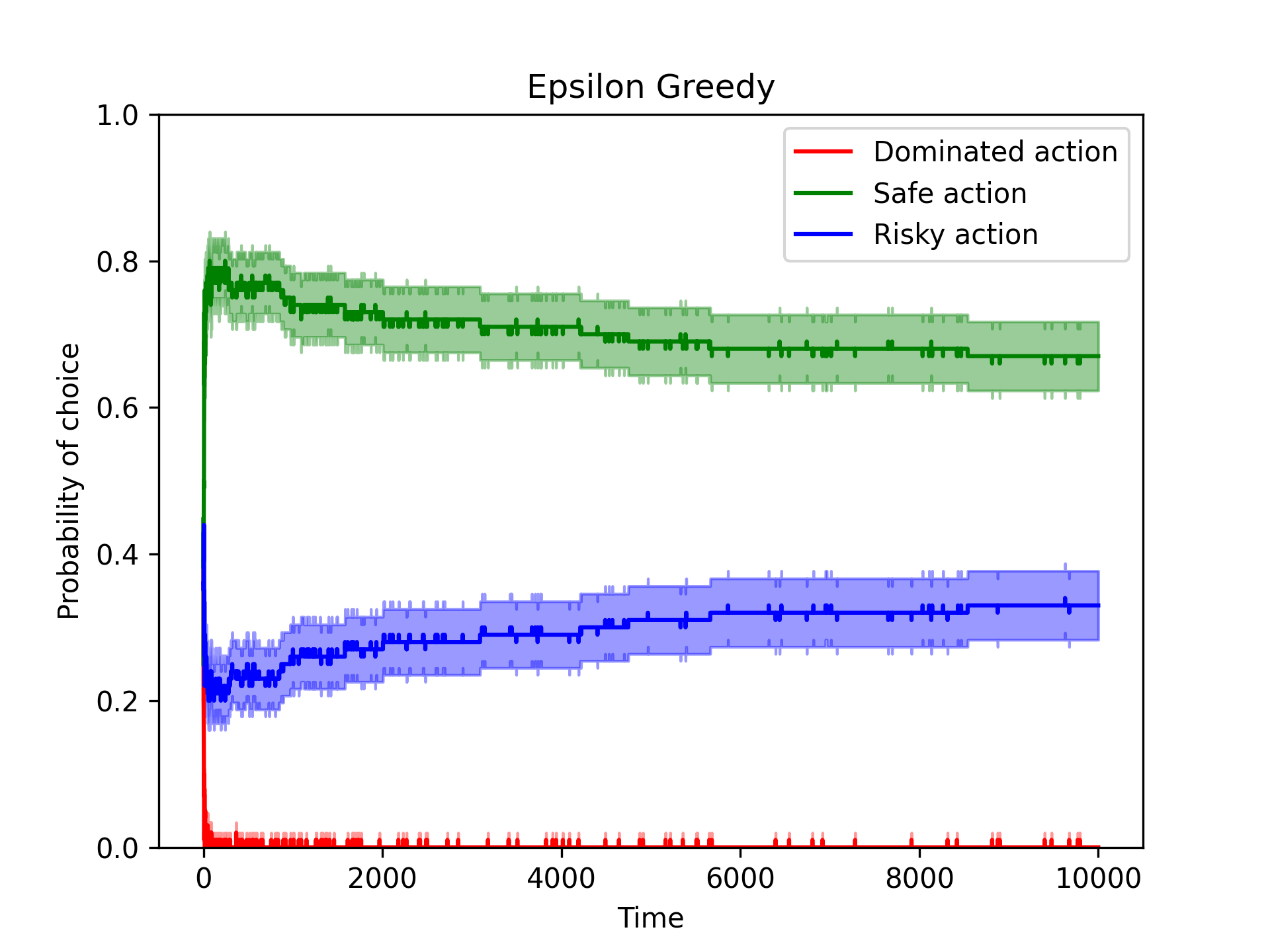}
\caption{$\varepsilon$-Greedy with a strictly better risky action. }
\label{fig:better_risky}
\end{subfigure}
\caption{Plots showing that $\varepsilon$-Greedy's risk aversion is quite general. In a), the dominated action has reward distribution $\mathbb{1}_{\{-1\}}$, the safe action has distribution $U[-0.25, 0.25]$, the riskier action has payoff distribution $U[-0.5, 0.5]$ while the riskiest action has payoff distribution $U[-1, 1]$. In b), the safe action has payoff distribution $\mathbb{1}_{\{0\}}$ while the risky action has payoff distribution $U[-1, 1.2]$. The dominated action has payoff distribution $\mathbb{1}_{\{-1\}}$. We set $\varepsilon_t = t^{-1}$. Note that the bands around the plot line are $90\%$ confidence intervals, with 100 independent runs.}
\label{fig:eps_extra}
\end{figure}

\begin{figure}[ht]
\centering
\includegraphics[width=.34\linewidth]{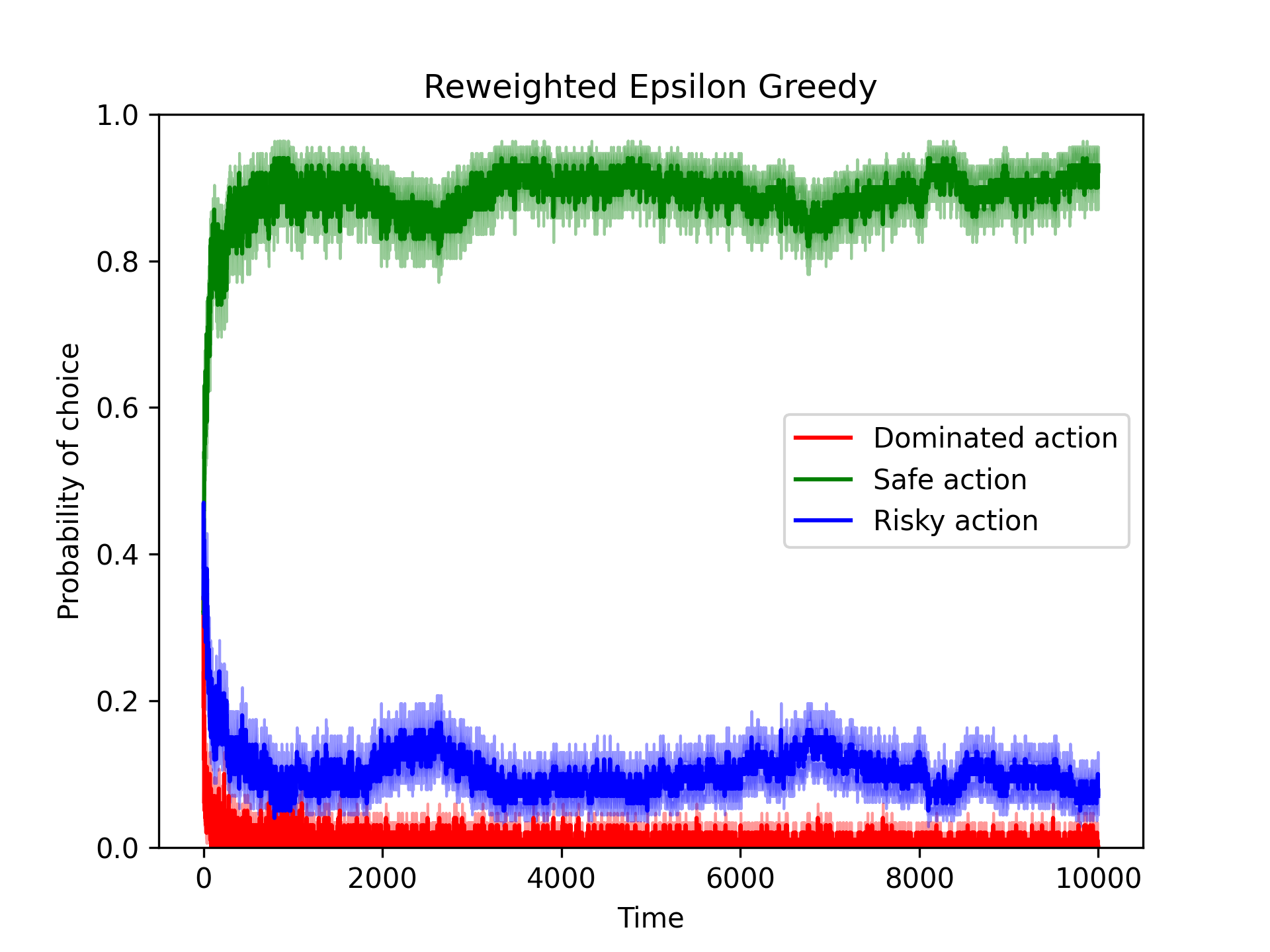}
\caption{Reweighted $\varepsilon$-Greedy where we add a third action with distribution $\mathbb{1}_{\{-1\}}$ to the setup of where the safe action has payoff distribution $\mathbb{1}_{\{0\}}$ while the risky action has payoff distribution $U[-1, 1]$. We set $\varepsilon_t = t^{-0.49}$. Note that the bands around the plot line are $90\%$ confidence intervals, with 100 independent runs.}
\label{fig:reweightedthreearms}
\end{figure}

\clearpage
\bibliographystyle{ACM-Reference-Format}
\bibliography{main.bib}

\clearpage
\appendix

\section{Regret analysis of algorithms}\label{sec:proof}
This section provides evidence that the algorithms we consider do not incur regret.
\begin{definition}
    An algorithm $\pi \colon $ \emph{is no-regret} or \emph{does not incur regret} if for all bandit problems such that $F_a$, $a \in A$ is sub-Gaussian and for all $a, a'\in A$, $\mu_a < \mu_{a'}$ implies 
    \[
    \P[a_t = a] \xrightarrow{t \to \infty} 0.
    \]
\end{definition}
We provide a proof sketch for the following statement, which implies that $\varepsilon$-Greedy does not incur regret.
\begin{lemma}\label{lem:epsgreedynoregret}
Let $\varepsilon_t \to 0$, $\sum_{t=1}^\infty \varepsilon_t = \infty$. Also let $a \in A$, $\mu_a < \max_{a \in A} \mu_a$. Then, for any $\delta > 0$, there is $t \in \N$ such that for all $t' \ge t$, 
\[
\pi_{t'} \le \delta.
\]
\end{lemma}
\begin{proof}[Proof Sketch]
Choose $t' \in \N$ such that for all $\tilde t \ge t'$, $\varepsilon_t \le \delta/2$. It is sufficient to show that for some $t''$, in exploitation steps,
\[
    \P [ \mu_a(t) - \mu_{a'}(t) \ge 0] \le \frac{\delta}{2}.
\]
By Hoeffding's inequality, with high probability in $t''$, both actions are chosen at least $\frac13 \sum_{t=1}^{t''} \varepsilon_t$ times. Conditional on this event, 
\[
\P [ \mu_a(t) - \mu_{a'}(t) \ge 0] \xrightarrow[t \to \infty]{}  0.
\]
In particular, 
\[
\P [ \mu_a(t) - \mu_{a'}(t) \ge 0] \le \frac{\delta}{2}.
\]
for some $t'' \in \N$ and any $\tilde t \ge t''$. Choosing $t \ge \max \{t', t''\}$ yields the claim.
\end{proof}
\begin{corollary}
$\varepsilon$-Greedy does not incur regret.
\end{corollary}
Next, we provide a proof sketch that optimistic $\varepsilon$-Greedy does not incur regret.
\begin{proposition}
Optimistic $\varepsilon$-Greedy does not incur regret.
\end{proposition}
\begin{proof}[Proof Sketch]
This proof is similar to the proof that Upper Confidence Bound does not incur regret (see, e.g., \cite{auer2002using}). The only difference between the Upper Confidence Bound algorithm and the Optimistic $\varepsilon$-Greedy is exploration that vanishes in the limit. It remains the case that the confidence bands are valid, and hence, there is a logarithmic upper bound for the probability that the bandit algorithm chooses a sub-optimal action in an exploitation step. As in the original proof of the Upper Confidence Band, this amounts to a sub-linearly growing probability of choosing a sub-optimal action.
\end{proof}
Finally, we give evidence that Reweighted $\varepsilon$-Greedy does not incur regret, first with a theoretical argument for a restricted class of instances and then for more general settings using an experiment.
\begin{proposition}\label{prop:reweightednoregret}
Let $\varepsilon_t = t^{-\frac12 + \kappa}$. Reweighted $\varepsilon$-Greedy does not incur regret if the best arm is deterministic and centered.
\end{proposition}
\begin{proof}[Proof Sketch]
We show that reweighted $\varepsilon$-Greedy does not incur regret for one deterministic, centered and one other arm of negative expected payoff. Using a union bound, this implies no-regretness for any number of negative-expectation arms against a deterministic centered arm.

Let $a, a' \in A$ be such that $0 = \mu_{a'} > \mu_{a}$. The statement is trivial if $F_{a}$ is deterministic. 

If $F_{a}$ is non-deterministic, observe that, as in the the proof of \autoref{thm:reweighted}
\[
Y_t = \sum_{1 \le t' \le t} \frac{r_{t'}}{\sqrt{\pi_t}}
\]
is a sufficient statistic for Reweighted $\varepsilon$-Greedy. It is sufficient to show that $\P[Y_t < 0] \to 1$ as $t \to \infty$. (As in \autoref{thm:epsgreedy}, $\P[Y_t = 0] \to 0$ as $t \to \infty$.) To show $\P[Y_t < 0] \to 1$ as $t \to \infty$, we condition on $Y_t$'s last crossing time $\tau_t$ through zero. As in the proof of \autoref{thm:reweighted}, we observe that $t - \tau_t$ grows large. We will show that conditional on a large $\tau_t$, the probability that $Y_t > 0$ is small, which implies the claim.

To this end, we observe the increment has a bias of $\frac{\mu_a}{\sqrt{1-\varepsilon_t/2}}$ and a standard deviation of 
\[
\frac{\sqrt{\var (F_a)}}{\sqrt{1-\frac{\varepsilon_t}{2}}},
\]
and that increments are independent (conditioning on the last crossing time). An application of a Chernoff bound shows that the probability that the walk is positive is small.
\end{proof}
We supplement this theoretical result with a simulation beyond the conditions of \autoref{fig:reweighted_noregret}. Figure \autoref{fig:reweighted_noregret} shows that even with a non-deterministic, non-centered action of highest expected payoff, Reweighted $\varepsilon$-Greedy appears not to incur regret.
\begin{figure}[ht]
\centering
\begin{subfigure}{.34\textwidth}
\centering
\includegraphics[width=.95\linewidth]{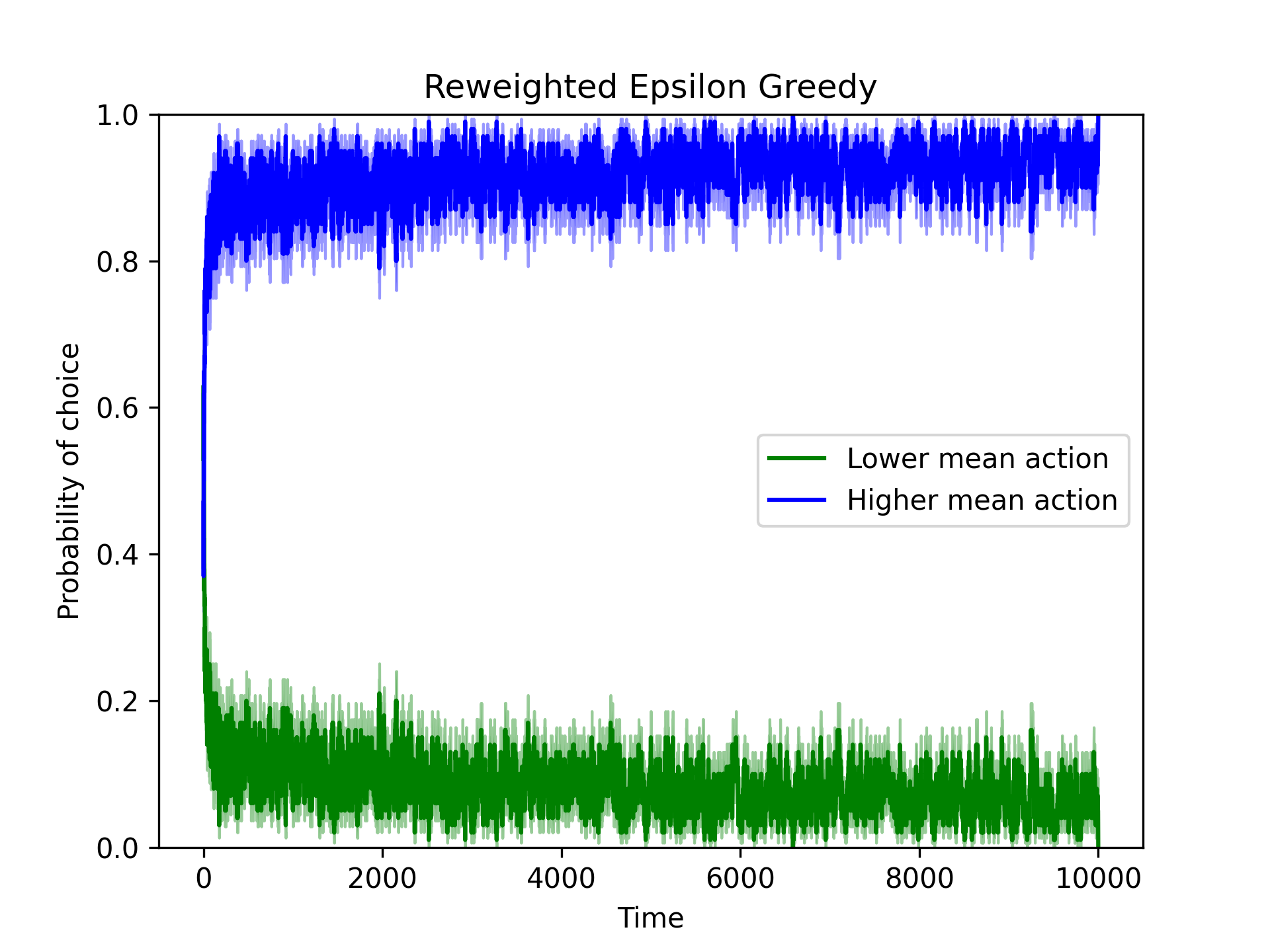}
\caption{Reweighted $\varepsilon$-Greedy for two actions, with a strictly better risky action.}
\label{fig:rew_three_arms}
\end{subfigure}%
\begin{subfigure}{.34\textwidth}
\includegraphics[width=.95\linewidth]{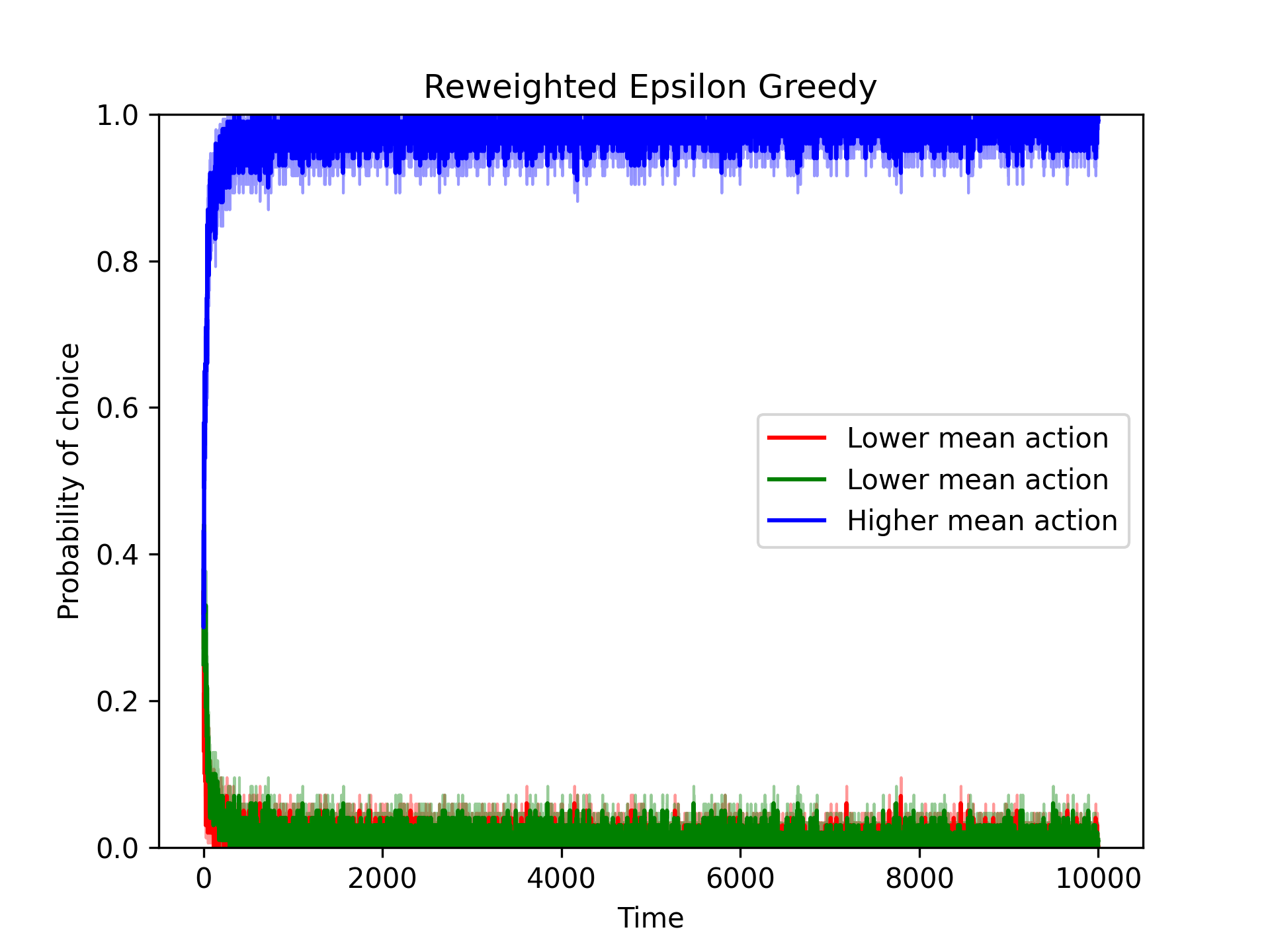}
\caption{Reweighted $\varepsilon$-Greedy for three actions, with a strictly better risky action.}
\label{fig:better_risky_rew}
\end{subfigure}
\caption{Plots illustrating that Reweighted $\varepsilon$-Greedy is no-regret. In a), the lower mean action has distribution $\Exp(1)$, the higher mean action has payoff distribution $\Exp(2) + N(0,1)$. In b) the lower mean actions have reward distribution $\mathbb{1}_{\{0.35\}}$ and $U[0.25, 0.75]$ while the higher mean action has payoff distribution $U[-1, 3]$. We set $\varepsilon_t = t^{-0.49}$. Note that the bands around the plot line are $90\%$ confidence intervals, with 100 independent runs.}
\label{fig:reweighted_noregret}
\end{figure}

\section{Additional simulations}\label{sec:add_sim}
Both $\varepsilon$-Greedy's and Reweighted $\varepsilon$-Greedy's reward estimates are biased (compare \cite[Section 11.2]{Lattimore2020}). It is natural to ask whether debiasing the reward estimates can address emergent risk preferences. Figure \ref{fig:unbiased} simulates a $\varepsilon$-Greedy with the unbiased sufficient statistic
\[
\mu_{a,d} (t) = \frac{1}{N_a(t)}\sum_{\substack{1 \le t' \le t \\ a_t =a}} \frac{r_t}{\pi_t}.
\]
The debiasing leads to risk \emph{affinity} as opposed to risk neutrality. One intuition for this is that the division by the square root of the choice probability in Reweighted $\varepsilon$-Greedy led to a correction that makes the algorithm exactly risk-neutral. Debiasing, which divides by the choice probability, a smaller number, leads to a stronger correction in the direction of risk affinity.
\begin{figure}[ht]
\centering
\includegraphics[width=.34\linewidth]{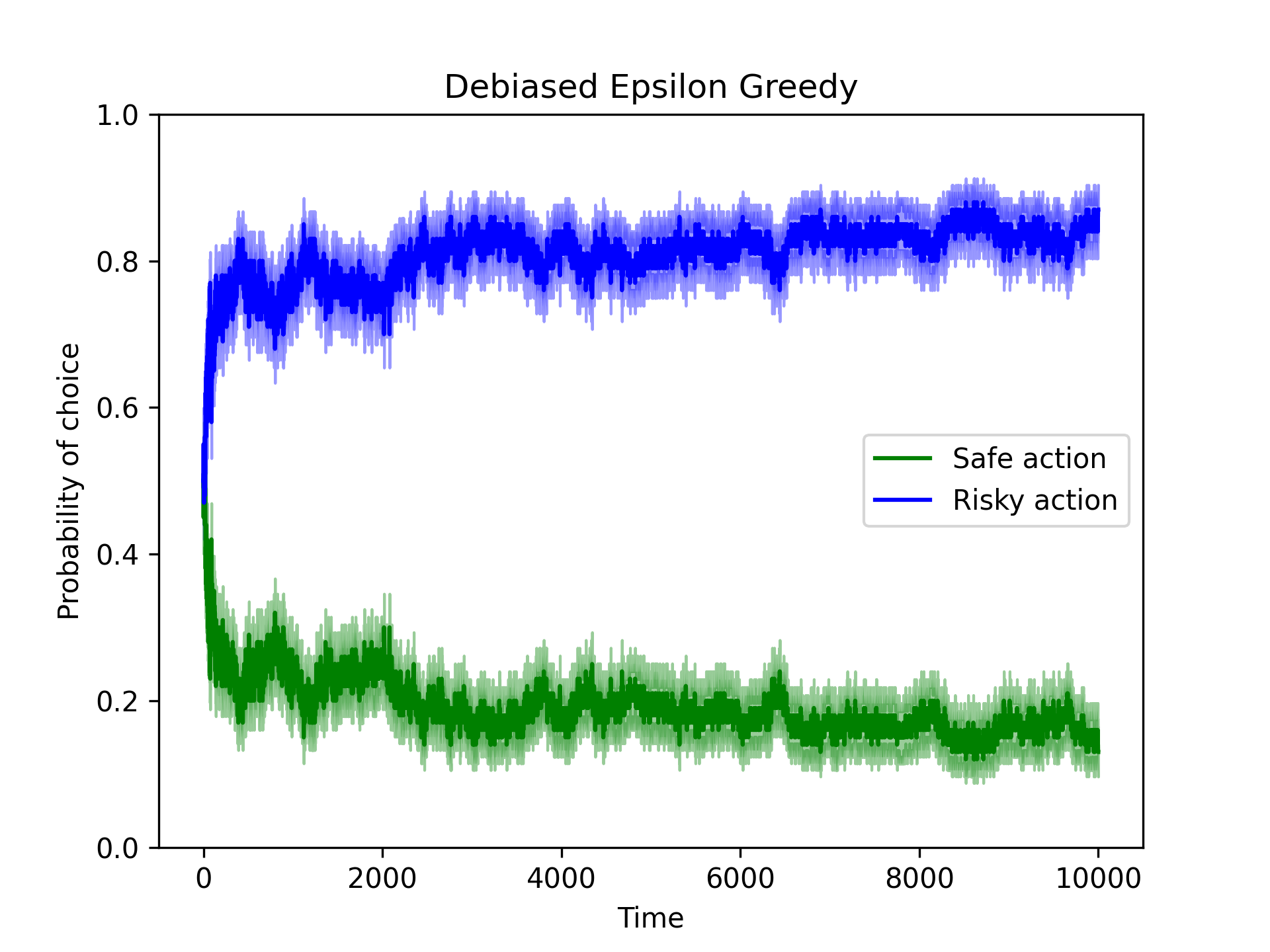}
\caption{$\varepsilon$-Greedy with a debiased reward estimate. The safe action has reward distribution $\mathbb{1}_{\{0\}}$ while the risky action has reward distribution $U[-1, 1]$. Note that the bands around the plot line are $90\%$ confidence intervals, with 100 independent runs.}
\label{fig:unbiased}
\end{figure}
\end{document}